\documentclass{llncs}

\usepackage{makeidx}

\usepackage[utf8x]{inputenc}
\usepackage[T1]{fontenc}
\usepackage{dashundergaps}

\usepackage[english]{babel}

\usepackage{graphicx} 
\graphicspath{{images/}}
\usepackage{subfigure}

\usepackage{color}

\usepackage{algorithm}
\usepackage{algorithmic}

\usepackage{booktabs}

\usepackage{amsmath}
\usepackage{amssymb}

\usepackage{url}

\DeclareMathOperator{\Var}{Var}

\begin{document}

\title{Random forests with random projections of the output space for high
       dimensional multi-label classification}
\author{Arnaud Joly \and Pierre Geurts \and Louis Wehenkel}
\institute{Dept. of EE \& CS \& GIGA-R\\University of Liège, Belgium}
\maketitle

\begin{abstract}
We adapt the idea of random projections applied to the output space, so as to
enhance tree-based ensemble methods in the context of multi-label
classification. We show how learning time complexity can be reduced without
affecting computational complexity and accuracy of predictions. We also show
that random output space projections may be used in order to reach different
bias-variance tradeoffs, over a broad panel of benchmark problems, and that
this may lead to improved accuracy while reducing significantly the
computational burden of the learning stage.
\end{abstract}

\section{Introduction} \label{sec:introduction}

Within supervised learning, the goal of multi-label classification is to train
models to annotate objects with a subset of labels taken from a set of
candidate labels. Typical applications include the determination of topics
addressed in a text document, the identification of object categories
present within an image, or the prediction of biological properties of a
gene. In many applications, the number of candidate labels may be very large, ranging from hundreds to
hundreds of thousands~\cite{DBLP:conf/www/AgrawalGPV13} and often
even exceeding the sample size \cite{dekel2010multiclass}. The very
large scale nature of the output space in such problems poses both statistical
and computational challenges that need to be specifically addressed.

A simple approach to solve multi-label classification problems, called
binary relevance, is to train independently a binary classifier for
each label. Several more complex schemes have however been proposed to
take into account the dependencies between the labels (see, e.g.
\cite{read2009,DBLP:conf/nips/HsuKLZ09,DBLP:conf/icml/DembczynskiCH10,DBLP:journals/tkde/TsoumakasKV11,NIPS2013_5083,zhou2012multi}).  In
the context of tree-based methods, one way is to train multi-output
trees
\cite{blockeel-1998,DBLP:conf/icml/GeurtsWd06,DBLP:journals/pr/KocevVSD13},
ie. trees that can predict multiple outputs at once. With respect to
single-output trees \cite{DBLP:books/wa/BreimanFOS84}, the score
measure used in multi-output trees to choose splits is taken as the
sum of the individual scores corresponding to the different labels
(e.g., variance reduction) and each leaf is labeled with a vector of
values, coding each for the probability of presence of one label. With
respect to binary relevance, the multi-output tree approach has the
advantage of building a single model for all labels. It can thus
potentially take into account label dependencies and reduce memory
requirements for the storage of the models. An extensive experimental
comparison \cite{DBLP:journals/pr/MadjarovKGD12} shows that this
approach compares favorably with other approaches, including non
tree-based methods, both in terms of accuracy and computing times.
In addition, multi-output trees inherit all intrinsic
advantages of tree-based methods, such as robustness to irrelevant
features, interpretability through feature importance scores, or
fast computations of predictions, that make them very attractive to
address multi-label problems. The computational complexity of
learning multi-output trees is however similar to that of the binary
relevance method. Both approaches are indeed $O(p d n\log n)$, where
$p$ is the number of input features, $d$ the number of candidate
output labels, and $n$ the sample size; this is a limiting factor when
dealing with large sets of candidate labels.

One generic approach to reduce computational complexity is to apply some
compression technique prior to the training stage to reduce the number of
outputs to a number $m$ much smaller than the total number $d$ of labels. A model
can then be trained to make predictions in the compressed output space and a
prediction in the original label space can be obtained by decoding the
compressed prediction. As multi-label vectors are typically very sparse, one
can expect a drastic dimensionality reduction by using appropriate
compression techniques. This idea has been explored for example in
\cite{DBLP:conf/nips/HsuKLZ09} using compressed sensing, 
and in \cite{NIPS2013_5083} using bloom filters, in both cases using regularized
linear models as base learners. This approach obviously reduces computing
times for training the model. At the prediction stage however, the predicted
compressed output needs to be decoded, which adds
computational cost and can also introduce further decoding
errors.

In this paper, we explore the use of random output space projections
for large-scale multi-label classification in the context of tree-based
ensemble methods. We first explore the idea proposed for linear models in
\cite{DBLP:conf/nips/HsuKLZ09} with random
forests: a (single) random projection of the multi-label vector to an
$m$-dimensional random subspace is computed and then a multi-output random
forest is grown based on score computations using the projected outputs. We
exploit however the fact that the
approximation provided by a tree ensemble is a weighted average of output
vectors from the training sample to avoid the decoding stage: at training time all leaf
labels are directly computed  in the original multi-label space. We show theoretically and
empirically that when $m$ is large enough, ensembles grown on such random
output spaces are equivalent to ensembles grown on the original output
space. When $d$ is large enough compared to $n$, this idea hence may reduce
computing times at the learning stage without affecting accuracy and
computational complexity of predictions.

Next, we propose to exploit the randomization inherent to the
projection of the output space as a way to obtain randomized trees in
the context of ensemble methods: each tree in the ensemble is thus grown from a
different randomly projected subspace of dimension $m$. As previously,
labels at leaf nodes are directly computed in the original output space to avoid the
decoding step. We show, theoretically, that this idea can lead to better
accuracy than the first idea and, empirically, that best results are obtained on many problems with very low
values of $m$, which leads to significant computing time reductions at the
learning stage. In addition, we study the interaction between input randomization (\`a la
Random Forests) and output randomization (through random projections), showing
that there is an interest, both in terms of predictive performance and in terms of
computing times, to optimally combine these two ways of randomization. All
in all, the proposed approach constitutes a very attractive way to address
large-scale multi-label problems with tree-based ensemble methods.

The rest of the paper is structured as follows: Section \ref{sec:background}
reviews properties of multi-output tree ensembles and of random projections;
Section \ref{sec:methods} presents the proposed algorithms and their
theoretical properties; Section  \ref{sec:experiments} provides the empirical
validations, whereas Section  \ref{sec:conclusions} discusses our work and
provides further research directions.

\section{Background}
\label{sec:background}

We denote by $\mathcal{X}$ an input space, and by $\mathcal{Y}$ an output
space; without loss of generality, we suppose that
$\mathcal{X} = \mathbb{R}^{p}$ (where $p$ denotes the number of input features),
and that $\mathcal{Y} = \mathbb{R}^{d}$ (where $d$ is the dimension of the
output space). We denote by $P_{{\cal X},{\cal Y}}$ the joint (unknown)
sampling density over $\mathcal{X} \times \mathcal{Y}$.

Given a learning sample $\left((x^i, y^i) \in
(\mathcal{X} \times \mathcal{Y})\right)_{i=1}^n$ of $n$ observations in the
form of input-output pairs, a supervised learning task is defined as
searching for a function $f^{*} : \mathcal{X} \rightarrow \mathcal{Y}$ in a
hypothesis space $\mathcal{H} \subset \mathcal{Y}^\mathcal{X}$ that minimizes
the expectation of some loss function $\ell : \mathcal{Y} \times \mathcal{Y}
\rightarrow \mathbb{R}$ over the joint distribution of input / output pairs: $f^{*} \in
\arg \min_{f \in \mathcal{H}} E_{P_{{\cal X},{\cal Y}}} \left\{ \ell(f(x), y) \right\}.
$

NOTATIONS: Superscript indices ($x^{i}, y^{i}$) denote (input, output) vectors
of an observation $i \in \{1, \ldots , n\}$. Subscript indices
(e.g. $x_{j}, y_{k}$) denote components of vectors.

\subsection{Multi-output tree ensembles}
\label{sec:algorithm-description}

A classification or a regression tree~\cite{DBLP:books/wa/BreimanFOS84}
is built using all the input-output pairs as follows: for
each node at which the subsample size
is greater or equal to a pre-pruning parameter $n_{\min}$, the best split
is chosen among the $p$ input features
combined with the selection of an optimal cut point.
The best sample split $(S_{r}, S_{l})$ of the local subsample $S$ minimizes the average
reduction of impurity
\begin{align}
\hspace*{-0.5mm}&\hspace*{-0.5mm}\Delta I((y^i)_{i \in S}, (y^i)_{i \in S_l}, (y^i)_{i \in S_r}) =
\hspace*{-0.5mm}&\hspace*{-0.5mm} I(\hspace*{-0.3mm}(y^i)_{i \in S}\hspace*{-0.3mm}) \hspace*{-0.9mm}- \hspace*{-0.9mm}\frac{|S_l|}{|S|} I(\hspace*{-0.3mm}(y^i)_{i \in S_l}\hspace*{-0.3mm})
\hspace*{-0.9mm}- \hspace*{-0.9mm}\frac{|S_r|}{|S|} I(\hspace*{-0.3mm}(y^i)_{i \in S_r}\hspace*{-0.3mm}).
\label{eq:impurity_reduction}
\end{align}
Finally, leaf statistics are obtained by aggregating the outputs of the samples reaching
that leaf.

In this paper, for multi-output trees, we use the sum of the variances of the
$d$ dimensions of the output vector as an impurity measure. It can be computed
by (see Appendix~A, in the supplementary material\footnote{
\url{static.ajoly.org/files/ecml2014-supplementary.pdf}})
\begin{align}\label{multi-var}
\Var((y^{i})_{i \in S})
&= \frac{1}{|S|} \sum_{i \in S} ||y^i - \frac{1}{|S|} \sum_{i\in S} y^i||^2, \\
&= \frac{1}{2|S|^2} \sum_{i\in S} \sum_{j\in S} ||y^i - y^j||^2.
\end{align}
\noindent Furthermore, we compute the vectors of output statistics by
component-wise averaging. Notice that, when the outputs are vectors of binary
class-labels (i.e. $y\in \{0,1\}^{d}$), as in multi-label classification, the
variance reduces to the so-called Gini-index, and the leaf statistics then
estimate a vector of conditional probabilities $P(y_{j} = 1 | x \in leaf)$,
from which a prediction $\hat{y}$ can be made by thresholding.

Tree-based ensemble methods build an ensemble of $t$ randomized trees.
Unseen samples are then predicted by aggregating the predictions of all $t$ trees.
Random Forests \cite{DBLP:journals/ml/Breiman01} build each tree on
a bootstrap copy of the learning sample \cite{DBLP:journals/ml/Breiman01} and by optimising
the split at each node over a locally generated random subset of size $k$  among
the $p$ input features.
Extra Trees \cite{DBLP:journals/ml/GeurtsEW06} use the complete learning sample
and optimize the split over a random subset of size $k$ of
the $p$ features combined with a random selection of cut points.
Setting the parameter $k$
to the number of input features $p$ allows to filter out irrelevant
features; larger $n_{\min}$ yields simpler trees possibly at the price
of higher bias, and the higher $t$ the smaller the variance of the resulting predictor.

\subsection{Random projections}
\label{sec:rp}

In this paper we apply the idea of random projections to samples of vectors of
the output space $\mathcal{Y}$. With this in mind, we recall the
Johnson-Lindenstrauss lemma (reduced to linear maps), while using our notations.

\begin{lemma}{Johnson-Lindenstrauss lemma~\cite{johnson1984extensions}}
\label{lemma:jl-lemma}
Given $\epsilon > 0$ and an integer $n$, let $m$ be a positive integer such
that  $m \geq 8 \epsilon^{-2} \ln {n}$. For any sample $(y^i)_{i=1}^{n}$ of $n$ points
in $\mathbb{R}^d$ there exists a matrix $\Phi \in \mathbb{R}^{m \times d}$
such that for all $i, j  \in \{1, \ldots , n\}$
\begin{equation}\label{eqn:js}
(1 \hspace*{-0.3mm} - \hspace*{-0.3mm}\epsilon) ||y^i \hspace*{-0.3mm}- \hspace*{-0.3mm}y^j||^2 \leq || \Phi y^i \hspace*{-0.3mm}- \hspace*{-0.3mm}\Phi y^j ||^2
                           \leq (1\hspace*{-0.3mm} +\hspace*{-0.3mm} \epsilon) || y^i \hspace*{-0.3mm}- \hspace*{-0.3mm}y^{j}||^2.
\end{equation}
\end{lemma}

Moreover, when $d$ is sufficiently large, several random matrices satisfy
(\ref{eqn:js}) with high probability. In particular, we can consider Gaussian
matrices which elements are drawn  {\em i.i.d.} in $\mathcal{N}(0, 1 / m)$, as
well as (sparse) Rademacher matrices which elements are drawn in
$\left\{ -\sqrt{\frac{s}{m}}, 0, \sqrt{\frac{s}{m}} \right\}$ with probability
$\left\{ \frac{1}{2s}, 1 - \frac{1}{s} ,\frac{1}{2s}\right\}$, where
$1 / s \in (0,1]$ controls the sparsity of
$\Phi$~\cite{DBLP:journals/jcss/Achlioptas03,DBLP:conf/kdd/LiHC06}.

Notice that if some $\Phi$ satisfies (\ref{eqn:js}) for the whole learning
sample, it obviously satisfies (\ref{eqn:js}) for any subsample that could
reach a node during regression tree growing. On the other hand, since we are
not concerned in this paper with the `reconstruction' problem, we do not need
to make any sparsity assumption `\`{a} la compressed sensing'.


\section{Methods}
\label{sec:methods}


We first present how we propose to exploit random projections to reduce the computational burden of learning single
multi-output trees in very high-dimensional output spaces. Then we present and compare two
ways to exploit this idea with ensembles of trees. Subsection~
\ref{sec:biasvar} analyses these two ways from the bias/variance point of view.

\subsection{Multi-output regression trees in randomly projected output spaces}
\label{sec:theoretical-analysis}

The multi-output single tree algorithm described in section~2 requires the computation
of the sum of variances in (\ref{multi-var}) at each tree node and for each candidate
split. When $\mathcal{Y}$ is very high-dimensional, this computation constitutes
the main computational bottleneck of the algorithm. We
thus propose to approximate variance computations by using random projections of the
output space. The multi-output regression tree algorithm is modified as
follows (denoting by $LS$ the  learning sample $((x^i,y^i))_{i=1}^n$):
\begin{itemize}
\item First, a projection matrix $\Phi$ of dimension $m\times d$ is
  randomly generated.
\item A new dataset $LS_m=((x^i,\Phi y^i))_{i=1}^n$ is constructed by projecting
  each learning sample output using the projection matrix $\Phi$.
\item A tree (structure) $\cal T$ is grown using the projected learning sample
  $LS_m$.
\item Predictions $\hat{y}$ at each leaf of $\cal T$ are computed using the
      corresponding outputs in the original output space.
\end{itemize}
The resulting tree is exploited in the standard way to make predictions:
an input vector $x$ is propagated through the tree until it
reaches a leaf from which a prediction $\hat{y}$ in the original output space is directly retrieved.

If $\Phi$ satisfies (\ref{eqn:js}), the following theorem shows that variance computed in the projected
subspace is an $\epsilon$-approximation of the variance computed over the original space.

\begin{theorem}
\label{thm:var-jl-lemma}
Given $\epsilon > 0$, a sample $(y^i)_{i=1}^{n}$ of $n$ points $y \in \mathbb{R}^d$, and a projection matrix $\Phi\in
\mathbb{R}^{m\times d}$ such that for all $i, j  \in \{1, \ldots , n\}$ condition (\ref{eqn:js}) holds, we have also:
\begin{equation}
(1 - \epsilon) \Var((y^i)_{i=1}^n)  \leq \Var((\Phi y^i)_{i=1}^n)
                                    \leq (1 + \epsilon) \Var((y^i)_{i=1}^n).
\end{equation}
\end{theorem}
\begin{proof}
See Appendix~B, supplementary material.
\end{proof}
As a consequence, any split score approximated from the randomly
projected output space will be $\epsilon$-close to the unprojected
scores in any subsample of the complete learning sample. Thus, if
condition (\ref{eqn:js}) is satisfied for a sufficiently small
$\epsilon$ then the tree grown from the projected data will be
identical to the tree grown from the original
data\footnote{Strictly speaking, this is only the case
when the optimum scores of test splits as computed over the
original output space are isolated, i.e. when there is only one
single best split, no tie.}.

For a given size $m$ of the projection subspace, the complexity is
reduced from $O(dn)$ to $O(mn)$ for the computation of one split score
and thus from $O(d p n\log n)$ to $O(m p n\log n)$ for the
construction of one full (balanced) tree, where one can expect $m$ to
be much smaller than $d$ and at worst of $O(\epsilon^{-2}
\log{n})$. The whole procedure requires to generate the projection
matrix and to project the training data. These two steps are
respectively $O(d m)$ and $O(n d m)$ but they can often be
significantly accelerated by exploiting the sparsity of the projection
matrix and/or of the original output data, and they are called only
once before growing the tree.

All in all, this means that when $d$ is sufficiently large, the random
projection approach may allow us to significantly reduce tree building
complexity from $O(d t p n \log n)$ to $O(m t p n \log
  n + t n d m)$, without impact on predictive accuracy (see section
4, for empirical results).

\subsection{Exploitation in the context of tree ensembles}

The idea developed in the previous section can be directly exploited  in the
context of ensembles of randomized
multi-output regression trees. Instead of building a single tree from the
projected learning sample $LS_m$, one can grow a randomized ensemble of them.
This ``shared subspace'' algorithm is described in pseudo-code in
Algorithm~\ref{alg:output-fix-subspace-tree-ensemble}.

\begin{algorithm}
\caption{Tree ensemble on a single shared subspace $\Phi$}
\label{alg:output-fix-subspace-tree-ensemble}

\begin{algorithmic}
\REQUIRE{$t$, the ensemble size}
\REQUIRE{$((x^i, y^i) \in (\mathbb{R}^p \times \mathbb{R}^d))_{i=1}^n$, the
         input-output pairs}
\REQUIRE{A tree building algorithm.}
\REQUIRE{A sub-space generator}

\STATE{Generate a sub-space $\Phi \in \mathbb{R}^{m \times d}$;}

\FOR{$j = 1$ to $t$}
    \STATE{Build a tree structure $\mathcal{T}_{j}$ using $((x^i, \Phi y^i))_{i=1}^n$;}
    \STATE{Label the leaves of $\mathcal{T}_{j}$ using $((x^i, y^i))_{i=1}^n$;}
    \STATE{Add the labelled tree $\mathcal{T}_{j}$ to the ensemble;}
\ENDFOR
\end{algorithmic}
\end{algorithm}

Another idea is to exploit the random
projections used so as to introduce a novel kind of diversity among the
different trees of an ensemble. Instead of building all the trees of the ensemble from a
same shared output-space projection, one could instead grow each tree in the ensemble from a
different output-space projection. Algorithm~\ref{alg:output-subspace-tree-ensemble}
implements this idea in pseudo-code. The randomization introduced by the output space
projection can of course be combined with any existing randomization scheme
to grow ensembles of trees. In this paper, we will consider the combination of random projections
with the randomizations already introduced in Random Forests and Extra Trees.
The interplay between these different randomizations will be
discussed theoretically in the next subsection by a bias/variance analysis and
empirically in Section~\ref{sec:experiments}. Note that while when looking at
single trees or shared ensembles, the size $m$ of the projected subspace should not be too small so
that condition (\ref{eqn:js}) is satisfied, the optimal value of $m$ when projections
are randomized at each tree is likely to be smaller, as suggested by the bias/variance
analysis in the next subsection.

\begin{algorithm}
\caption{Tree ensemble with individual subspaces $\Phi_{j}$}
\label{alg:output-subspace-tree-ensemble}

\begin{algorithmic}
\REQUIRE{$t$, the ensemble size}
\REQUIRE{$((x^i, y^i) \in (\mathbb{R}^p \times \mathbb{R}^d))_{i=1}^n$, the
         input-output pairs}
\REQUIRE{A tree building algorithm.}
\REQUIRE{A sub-space generator}

\FOR{$j = 1$ to $t$}
    \STATE{Generate a sub-space $\Phi_{j} \in \mathbb{R}^{m \times d}$;}
    \STATE{Build a tree structure $\mathcal{T}_{j}$ using $((x^i, \Phi_{j} y^i))_{i=1}^n$;}
    \STATE{Label the leaves  of $\mathcal{T}_{j}$  using $((x^i, y^i))_{i=1}^n$;}
    \STATE{Add the labelled tree $\mathcal{T}_{j}$  to the ensemble;}
\ENDFOR
\end{algorithmic}
\end{algorithm}

From the computational point of view, the main difference between these two
ways of transposing random-output projections to ensembles of trees is that in
the case of Algorithm~\ref{alg:output-subspace-tree-ensemble}, the generation
of the projection matrix $\Phi$ and the computation of projected outputs is
carried out $t$ times, while it is done only once for the case of
Algorithm~\ref{alg:output-fix-subspace-tree-ensemble}. These aspects will
be empirically evaluated in Section~4.

\subsection{Bias/variance analysis}\label{sec:biasvar}

In this subsection, we adapt the bias/variance analysis carried out in
\cite{DBLP:journals/ml/GeurtsEW06} to take into account random output
projections. The details of the derivations are
reported in Appendix~C (supplementary material).

Let us denote by $f(.;ls,\phi,\epsilon):{\cal X}\rightarrow \mathbb{R}^d$ a
single multi-output tree obtained from a projection matrix $\phi$
(below we use $\Phi$ to denote the corresponding random variable), where
$\epsilon$ is the value of a random variable $\varepsilon$ capturing the random
perturbation scheme used to build this tree (e.g., bootstrapping and/or random
input space selection). The square error of this model at some point $x\in{\cal X}$ is defined by:
$$Err(f(x;ls,\phi,\epsilon))\stackrel{\text{def}}{=}E_{Y|x}\{||Y-f(x;ls,\phi,\epsilon\})||^2\},$$
and its average can decomposed in its residual error, (squared) bias, and
variance terms denoted:
$$E_{LS,\Phi,\varepsilon}\{Err(f(x;LS,\Phi,\varepsilon))\}=\sigma^2_R(x)+B^2(x)+V(x)$$
where the variance term $V(x)$ can be further decomposed as the sum of the
following three terms:
\begin{eqnarray*}
V_{LS}(x)&=&\Var_{LS}\{E_{\Phi,\varepsilon|LS}\{f(x;LS,\Phi,\varepsilon)\}\}\\
V_{Algo}(x)&=&E_{LS}\{E_{\Phi|LS}\{\Var_{\varepsilon|LS,\Phi}\{f(x;LS,\Phi,\varepsilon)\}\}\},\\
V_{Proj}(x)&=&E_{LS}\{\Var_{\Phi|LS}\{E_{\varepsilon|LS,\Phi}\{f(x;LS,\Phi,\varepsilon)\}\}\},
\end{eqnarray*}
that measure errors due to the randomness of, respectively, the learning
sample, the tree algorithm, and the output space projection (Appendix~C, supplementary material).

Approximations computed respectively by algorithms 1 and 2 take the following forms:\vspace*{-2mm}
\begin{itemize}
\item $f_1(x;ls,\epsilon^t,\phi)=\frac{1}{t}\sum_{i=1}^t f(x;ls,\phi, \epsilon_i)$
\item $f_2(x;ls,\epsilon^t,\phi^t)=\frac{1}{t}\sum_{i=1}^t f(x;ls,\phi_i,\epsilon_i),$\vspace*{-2mm}
\end{itemize}
where $\epsilon^t=(\epsilon_1,\ldots,\epsilon_{t})$ and
$\phi^t=(\phi_1,\ldots,\phi_t)$ are vectors of i.i.d. values of the random
variables $\varepsilon$ and $\Phi$ respectively.

We are interested in comparing the average errors of these two algorithms,
where the average is taken over all random parameters (including the
learning sample). We show (Appendix~C) that these can be decomposed as
follows:
\begin{eqnarray*}
&&\hspace*{-5mm}E_{LS,\Phi,\varepsilon^t}\{Err(f_1(x;LS,\Phi,\varepsilon^t))\}\\
&&\hspace{-0.2cm}=\sigma^2_R(x)+B^2(x)+V_{LS}(x)+\frac{V_{Algo}(x)}{t}+V_{Proj}(x),\\
&&\hspace*{-5mm}E_{LS,\Phi^t,\varepsilon^t}\{Err(f_2(x;LS,\Phi^t,\varepsilon^t))\}\\
&&\hspace{-0.2cm} =\sigma^2_R(x)+B^2(x)+V_{LS}(x)+\frac{V_{Algo}(x)+V_{Proj}(x)}{t}.
\end{eqnarray*}
From this result, it is hence clear that Algorithm 2 can not be worse, on the
average, than Algorithm 1. If the additional computational burden needed to
generate a different random projection for each tree is not problematic, then
Algorithm 2 should always be preferred to Algorithm 1.

For a fixed level of tree randomization ($\varepsilon$), whether the additional
randomization brought by random projections could be beneficial in terms of
predictive performance remains an open question that will be addressed
empirically in the next section. Nevertheless, with respect to an ensemble
grown from the original output space, one can expect that the
output-projections will always increase the bias term, since they disturb the algorithm
in its objective of reducing the errors on the learning sample. For small
values of $m$, the average error will therefore decrease (with a sufficiently
large number $t$ of trees) only if the increase in bias is
compensated by a decrease of variance.

The value of $m$, the dimension of the projected subspace, that will lead to the
best tradeoff between bias and variance will hence depend both on the level of tree
randomization and on the learning problem. The more (resp. less)
tree~randomization, the higher (resp. the lower) could be the optimal value of $m$,
since both randomizations affect bias and variance in the same direction.

\section{Experiments} \label{sec:experiments}

\subsection{Accuracy assessment protocol} \label{sec:experimental-protocol}

We assess the accuracy of the predictors for multi-label classification on a
test sample (TS) by the  ``Label Ranking Average
Precision (LRAP)'' \cite{DBLP:journals/pr/MadjarovKGD12}, expressed by
\begin{equation}
\text{LRAP}( \hat{f})
= \frac{1}{|TS|} \sum_{i\in TS} \frac{1}{|y^i|} \hspace*{-1mm}\sum_{j \in \{k : y^i_{k}=1\}} \hspace*{-1mm}
\frac{|\mathcal{L}_j^{i}(y^i)|}{|\mathcal{L}_j^{i}(1_d)|},
\end{equation}
\noindent
where
$\hat{f}(x^i)_j$ is the probability (or the score) associated to the label $j$
by the learnt model $\hat{f}$ applied to $x^i$, $1_d$ is a
$d$-dimensional row vector of ones, and
$$\mathcal{L}^i_{j}(q) = \left\{ k : q_{k}=1 \mbox{~and~} \hat{f}(x^i)_k
\geq \hat{f}(x^i)_j\right\}.$$
Test samples without any relevant labels (i.e. with $|y^i| = 0$) were discarded
prior to computing the average
precision. The best possible average precision is thus 1. Notice that we use
indifferently the notation $|\cdot |$ to express the cardinality of a set or
the $1$-norm of a vector.

\begin{figure}[htb]
\centering
\includegraphics[width=0.45\textwidth]{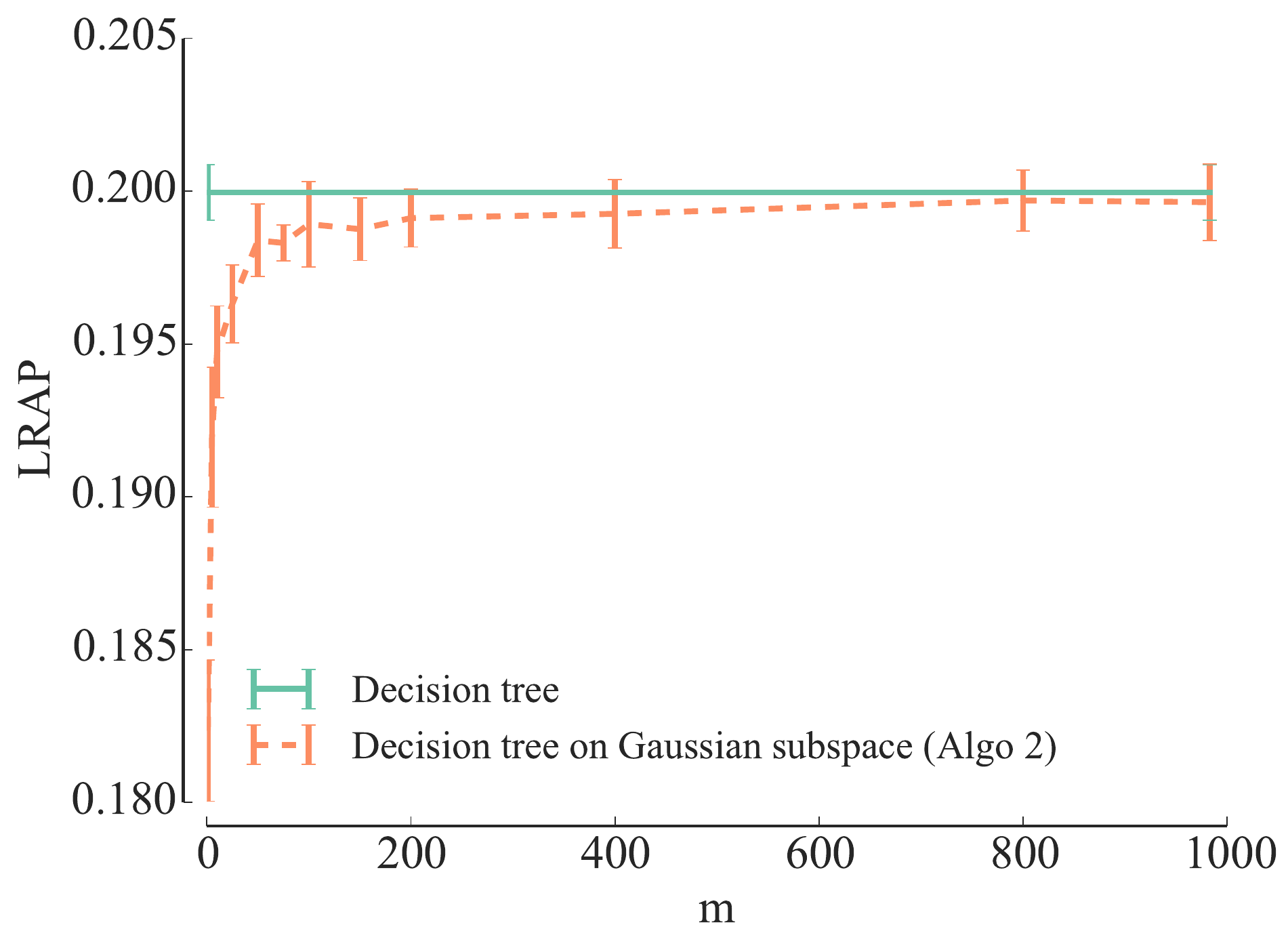}
\includegraphics[width=0.45\textwidth]{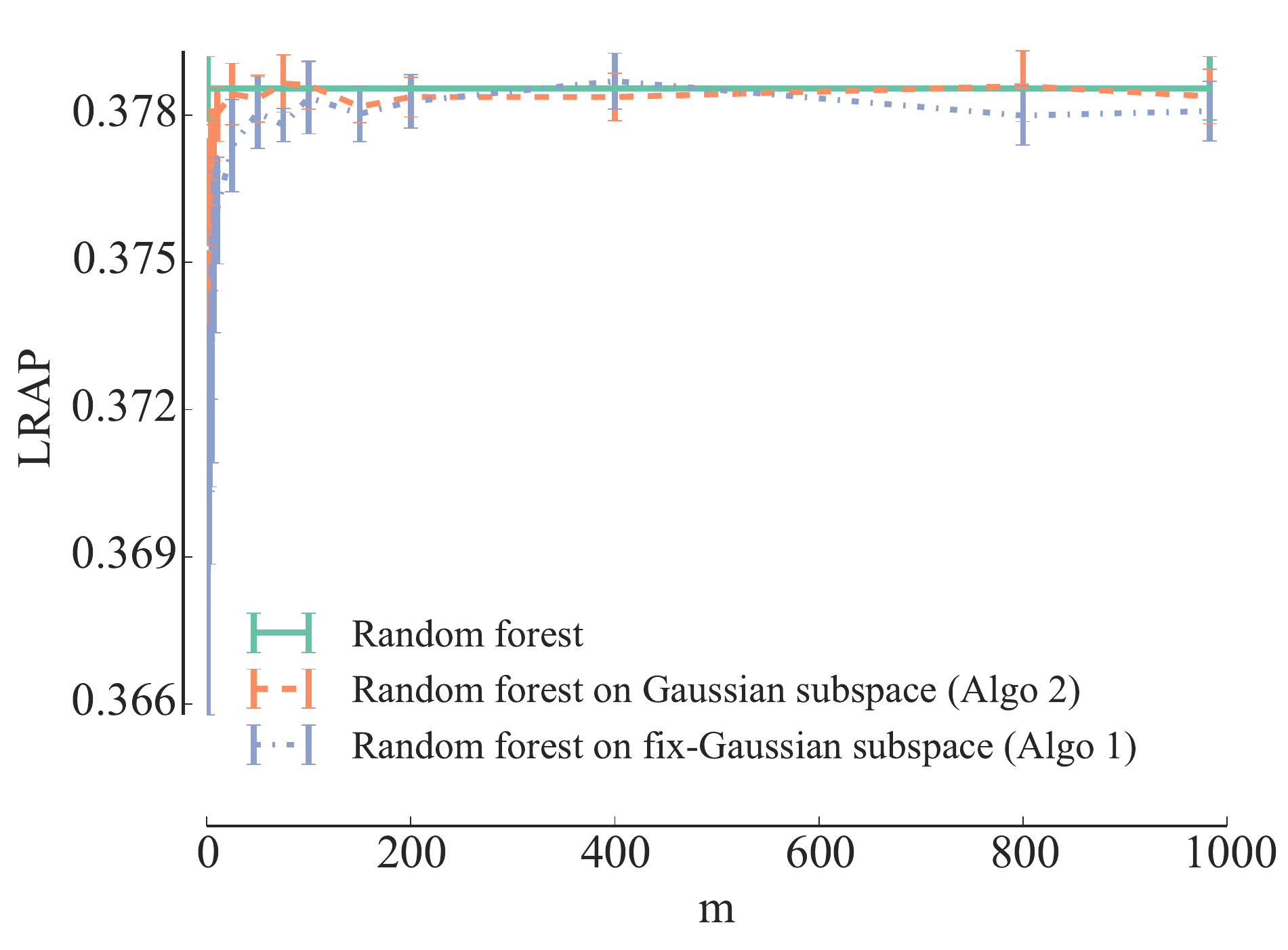}
\caption{Models built for the ``Delicious'' dataset ($d=983$) for growing
         numbers $m$ of Gaussian projections. Left: single unpruned CART
         trees ($n_{\min}=1$); Right: Random Forests ($k=\sqrt{p}$, $t=100$,
         $n_{\min}=1$). The curves represent average values (and standard
         deviations) obtained from 10 applications of the randomised algorithms
         over a same single $LS/TS$ split.}
\label{fig:delicious_vs_n_proj}
\end{figure}

\begin{table}[ht]
\centering
\caption{High output space compression ratio is possible, with no or negligible
  average precision reduction ($t=100$, $n_{\min}=1$, $k=\sqrt{p}$). Each
  dataset has $n_{LS}$ training samples, $n_{TS}$ testing samples, $p$ input
  features and $d$ labels. Label ranking average precisions are displayed in
  terms of their mean values and standard deviations over 10 random $LS/TS$
  splits, or over the 10 folds of cross-validation. Mean scores in the last
  three columns are underlined if they show a difference with respect to the
  standard Random Forests of more than one standard deviation.}
\renewcommand{\tabcolsep}{1.5mm}
\rotatebox{90}{
\begin{tabular}{@{} l rrrr r lll ll ll ll  @{}}
\toprule
\multicolumn{5}{c}{Datasets}                                        & \multicolumn{2}{c}{Random}      & \multicolumn{6}{c}{Random Forests on Gaussian sub-space } \\
\cmidrule(r){1-5} \cmidrule(r){8-9}\cmidrule(r){10-11} \cmidrule(r){12-13}
Name     & $n_{LS}$ & $n_{TS}$ &      $p$ &        $d$ &      \multicolumn{2}{c}{Forests}              & \multicolumn{2}{c}{$m=1$} & \multicolumn{2}{c}{$m\hspace*{-0.9mm}=\hspace*{-0.9mm}\lfloor\hspace*{-0.4mm} 0.5\hspace*{-0.8mm} +\hspace*{-0.8mm} \ln{d}\hspace*{-0.4mm}\rfloor$} & \multicolumn{2}{c}{$m=d$}\\
\midrule
emotions            &      391 &      202 &       72 &         6 & $0.800   $ &$\hspace*{-1.5mm} \pm 0.014$ & {$0.800$} & $\hspace*{-1.5mm} \pm 0.010 $ & $0.810  $ & $\hspace*{-1.5mm} \pm 0.014$ & $0.810  $ & $\hspace*{-1.5mm} \pm 0.016$ \\
scene               &     1211 &     1196 &     2407 &         6 & $0.870 $ & $\hspace*{-1.5mm} \pm 0.003$ & {$0.875$} & $\hspace*{-1.5mm} \pm 0.007$ & $0.872 $ & $\hspace*{-1.5mm} \pm 0.004$ & $0.872 $ & $\hspace*{-1.5mm} \pm 0.004$ \\
yeast               &     1500 &      917 &      103 &        14 & $0.759 $ & $\hspace*{-1.5mm} \pm 0.008$ & \dotuline{$ 0.748 $} & $\hspace*{-1.5mm} \pm 0.006$ & {$0.755$} & $\hspace*{-1.5mm} \pm 0.004$ & $0.758 $ & $\hspace*{-1.5mm} \pm 0.005$ \\
tmc2017             &    21519 &     7077 &    49060 &        22 & $0.756 $ & $\hspace*{-1.5mm} \pm 0.003$ & \dotuline{$ 0.741 $} & $\hspace*{-1.5mm} \pm 0.003$ & \dotuline{$ 0.748 $} & $\hspace*{-1.5mm} \pm 0.003$ & $0.757 $ & $\hspace*{-1.5mm} \pm 0.003$ \\
genbase             &      463 &      199 &     1186 &        27 & $0.992 $ & $\hspace*{-1.5mm} \pm 0.004$ & {$0.994$} & $\hspace*{-1.5mm} \pm 0.002$ & $0.994 $ & $\hspace*{-1.5mm} \pm 0.004$ & $0.993 $ & $\hspace*{-1.5mm} \pm 0.004$ \\
reuters             &     2500 &     5000 &    19769 &        34 & $0.865 $ & $\hspace*{-1.5mm} \pm 0.004$ & {$0.864$} & $\hspace*{-1.5mm} \pm 0.003$ & $0.863 $ & $\hspace*{-1.5mm} \pm 0.004$ & $0.862 $ & $\hspace*{-1.5mm} \pm 0.004$ \\
medical             &      333 &      645 &     1449 &        45 & $0.848 $ & $\hspace*{-1.5mm} \pm 0.009$ & \dotuline{$ 0.836 $} & $\hspace*{-1.5mm} \pm 0.011$ & {$0.842$} & $\hspace*{-1.5mm} \pm 0.014$ & $0.841 $ & $\hspace*{-1.5mm} \pm 0.009$ \\
enron               &     1123 &      579 &     1001 &        53 & $0.683 $ & $\hspace*{-1.5mm} \pm 0.009$ & {$0.680$} & $\hspace*{-1.5mm} \pm 0.006$ & $0.685 $ & $\hspace*{-1.5mm} \pm 0.009$ & $0.686 $ & $\hspace*{-1.5mm} \pm 0.008$ \\
mediamill           &    30993 &    12914 &      120 &       101 & $0.779 $ & $\hspace*{-1.5mm} \pm 0.001$ & \dotuline{$ 0.772 $} & $\hspace*{-1.5mm} \pm 0.001$ & {$0.777$} & $\hspace*{-1.5mm} \pm 0.002$ & $0.779 $ & $\hspace*{-1.5mm} \pm 0.002$ \\
Yeast-GO            &     2310 &     1155 &     5930 &       132 & $0.420  $ & $\hspace*{-1.5mm} \pm 0.010 $ & \dotuline{$ 0.353 $} & $\hspace*{-1.5mm} \pm 0.008$ & \dotuline{$ 0.381 $} & $\hspace*{-1.5mm} \pm 0.005$ & $0.420  $ & $\hspace*{-1.5mm} \pm 0.010$ \\
bibtex              &     4880 &     2515 &     1836 &       159 & $0.566 $ & $\hspace*{-1.5mm} \pm 0.004$ & \dotuline{$ 0.513 $} & $\hspace*{-1.5mm} \pm 0.006$ & \dotuline{$ 0.548 $} & $\hspace*{-1.5mm} \pm 0.007$ & $0.564 $ & $\hspace*{-1.5mm} \pm 0.008$ \\
CAL500              &      376 &      126 &       68 &       174 & $0.504 $ & $\hspace*{-1.5mm} \pm 0.011$ & {$0.504$} & $\hspace*{-1.5mm} \pm 0.004$ & $0.506 $ & $\hspace*{-1.5mm} \pm 0.007$ & $0.502 $ & $\hspace*{-1.5mm} \pm 0.010$ \\
WIPO                &     1352 &      358 &    74435 &       188 & $0.490  $ & $\hspace*{-1.5mm} \pm 0.010 $ & \dotuline{$ 0.430  $} & $\hspace*{-1.5mm} \pm 0.010 $ & \dotuline{$ 0.460  $} & $\hspace*{-1.5mm} \pm 0.010 $ & $0.480  $ & $\hspace*{-1.5mm} \pm 0.010$ \\
EUR-Lex (subj.)     &    19348 &    10-cv &     5000 &       201 & $0.840  $ & $\hspace*{-1.5mm} \pm 0.005$ & \dotuline{$ 0.814 $} & $\hspace*{-1.5mm} \pm 0.004$ & \dotuline{$ 0.828 $} & $\hspace*{-1.5mm} \pm 0.005$ & $0.840  $ & $\hspace*{-1.5mm} \pm 0.004$ \\
bookmarks           &    65892 &    21964 &     2150 &       208 & $0.453$ & $\hspace*{-1.5mm} \pm 0.001$& \dotuline{$ 0.436 $} & $\hspace*{-1.5mm} \pm 0.002$ & \dotuline{$ 0.445 $} & $\hspace*{-1.5mm} \pm 0.002$ & $0.453 $ & $\hspace*{-1.5mm} \pm 0.002$ \\
diatoms             &     2065 &     1054 &      371 &       359 & $0.700   $ & $\hspace*{-1.5mm} \pm 0.010 $ & \dotuline{$ 0.650  $} & $\hspace*{-1.5mm} \pm 0.010 $ & \dotuline{$ 0.670  $} & $\hspace*{-1.5mm} \pm 0.010$  & $0.710  $ & $\hspace*{-1.5mm} \pm 0.020$ \\
corel5k             &     4500 &      500 &      499 &       374 & $0.303 $ & $\hspace*{-1.5mm} \pm 0.012$ & {$0.309$} & $\hspace*{-1.5mm} \pm 0.011$ & $0.307 $ & $\hspace*{-1.5mm} \pm 0.011$ & $0.299 $ & $\hspace*{-1.5mm} \pm 0.013$ \\
EUR-Lex (dir.)      &    19348 &    10-cv &     5000 &       412 & $0.814 $ & $\hspace*{-1.5mm} \pm 0.006$ & \dotuline{$ 0.782 $} & $\hspace*{-1.5mm} \pm 0.008$ & \dotuline{$ 0.796 $} & $\hspace*{-1.5mm} \pm 0.009$ & $0.813 $ & $\hspace*{-1.5mm} \pm 0.007$ \\
SCOP-GO             &     6507 &     3336 &     2003 &       465 & $0.811 $ & $\hspace*{-1.5mm} \pm 0.004$ & {$0.808$} & $\hspace*{-1.5mm} \pm 0.005$ & $0.811 $ & $\hspace*{-1.5mm} \pm 0.004$ & \dotuline{$ 0.806 $} & $\hspace*{-1.5mm} \pm 0.004$ \\
delicious           &    12920 &     3185 &      500 &       983 & $0.384 $ & $\hspace*{-1.5mm} \pm 0.004$ & {$0.381$} & $\hspace*{-1.5mm} \pm 0.003$ & $0.382 $ & $\hspace*{-1.5mm} \pm 0.002$ & $0.383 $ & $\hspace*{-1.5mm} \pm 0.004$ \\
drug-interaction    &     1396 &      466 &      660 &      1554 & $0.379 $ & $\hspace*{-1.5mm} \pm 0.014$ & {$0.384$} & $\hspace*{-1.5mm} \pm 0.009$ & $0.378 $ & $\hspace*{-1.5mm} \pm 0.013$ & $0.367 $ & $\hspace*{-1.5mm} \pm 0.016$ \\
protein-interaction &     1165 &      389 &      876 &      1862 & $0.330  $ & $\hspace*{-1.5mm} \pm 0.015$ & {$0.337$} & $\hspace*{-1.5mm} \pm 0.016$ & $0.337 $ & $\hspace*{-1.5mm} \pm 0.017$ & $0.335 $ & $\hspace*{-1.5mm} \pm 0.014$ \\
Expression-GO       &     2485 &      551 &     1288 &      2717 & $0.235 $ & $\hspace*{-1.5mm} \pm 0.005$ & \dotuline{$ 0.211 $} & $\hspace*{-1.5mm} \pm 0.005$ & \dotuline{$ 0.219 $} & $\hspace*{-1.5mm} \pm 0.005$ & $0.232 $ & $\hspace*{-1.5mm} \pm 0.005$ \\
EUR-Lex (desc.)     &    19348 &    10-cv &     5000 &      3993 & $0.523 $ & $\hspace*{-1.5mm} \pm 0.008$ & \dotuline{$ 0.485 $} & $\hspace*{-1.5mm} \pm 0.008$ & \dotuline{$ 0.497 $} & $\hspace*{-1.5mm} \pm 0.009$ & $0.523    $ & $\hspace*{-1.5mm} \pm 0.007   $ \\
\bottomrule
\end{tabular}
}
\label{tab:random-forest}
\end{table}

\subsection{Effect of the size $m$ of the Gaussian output space}
\label{sec:empirical-convergence-delicious}

To illustrate the behaviour of our algorithms, we first focus on the
``Delicious''  dataset \cite{tsoumakas2008effective}, which has a large number of labels ($d=983$), of input
features ($p=500$), and of training ($n_{LS} = 12920$) and testing
($n_{TS}=3185$) samples.

The left part of figure~\ref{fig:delicious_vs_n_proj} shows, when Gaussian
output-space projections are combined with the standard CART algorithm
building a single tree, how the precision converges
(cf Theorem~\ref{thm:var-jl-lemma}) when $m$ increases towards $d$.
We observe that in this case, convergence is reached around $m=200$
at the expense of a slight decrease of accuracy, so that a
compression factor of about 5 is possible with respect to the original output
dimension $d=983$.

The right part of figure~\ref{fig:delicious_vs_n_proj} shows, on the same
dataset, how the method behaves when combined with Random Forests. Let us first
notice that the Random Forests  grown on the original output space (green line)
are significantly more accurate than the single trees, their accuracy being
almost twice as high.
We also observe that Algorithm~\ref{alg:output-subspace-tree-ensemble}
(orange curve) converges much more rapidly than
Algorithm~\ref{alg:output-fix-subspace-tree-ensemble} (blue curve) and
slightly outperforms the Random Forest grown on the original output space.
It needs only about $m=25$ components to converge, while
Algorithm~\ref{alg:output-fix-subspace-tree-ensemble} needs about $m=75$ of
them. These results are in accordance with the analysis of
Section~\ref{sec:biasvar}, showing that
Algorithm~\ref{alg:output-subspace-tree-ensemble} can't be inferior to
Algorithm~\ref{alg:output-fix-subspace-tree-ensemble}. In the rest of this
paper we will therefore focus on Algorithm~\ref{alg:output-subspace-tree-ensemble}.

\subsection{Systematic analysis over 24 datasets}
\label{sec:empirical-convergence-jl}

To assess our methods, we have collected 24 different multi-label
classification datasets from the literature (see Section~D of the supplementary
material, for more information and bibliographic references to these datasets)
covering a broad spectrum of application domains and ranges of the output
dimension ($d \in [6 ; 3993]$, see Table \ref{tab:random-forest}). For 21 of
the datasets, we made experiments where the dataset is split randomly into a
learning set of size $n_{LS}$, and a test set of size $n_{TS}$, and are
repeated 10 times (to get average precisions and standard deviations), and for
3 of them we used a ten-fold cross-validation scheme (see Table
\ref{tab:random-forest}).

Table~\ref{tab:random-forest} shows our results on the 24 multi-label datasets,
by comparing Random Forests learnt on the original output space with those
learnt by Algorithm~\ref{alg:output-subspace-tree-ensemble} combined with
Gaussian subspaces of size $m\in \left\{1, d, \ln{d}\right\}$\footnote{$\ln d$
is rounded to the nearest integer value; in Table~\ref{tab:random-forest} the
values of $\ln{d}$ vary between 2 for $d=6$ and 8 for $d=3993$.}. In these
experiments, the three parameters of Random Forests are set respectively to
$k=\sqrt{p}$, $n_{\min}=1$ (default values, see
\cite{DBLP:journals/ml/GeurtsEW06}) and $t=100$ (reasonable computing budget).
Each model is learnt ten times on a different shuffled train/testing split,
except for the 3 EUR-lex datasets where we kept the original 10 folds of
cross-validation.

We observe that for all datasets (except maybe SCOP-GO), taking $m=d$ leads to
a similar average precision to the standard Random Forests, i.e. no difference
superior to one standard deviation of the error.  On 11 datasets, we see that
$m=1$ already yields a similar average precision (values not underlined in
column $m=1$). For the 13 remaining datasets, increasing $m$ to $\ln{d}$
significantly decreases the gap with the Random Forest baseline and 3 more
datasets reach this baseline. We also observe that on several datasets such as
``Drug-interaction'' and ``SCOP-GO'', better performance on the Gaussian
subspace is attained with high output randomization
($m=\left\{1,\ln{d}\right\}$) than with $m=d$. We thus conclude that the
optimal level of output randomization (i.e. the optimal value of the ratio
$m/d$) which maximizes accuracy performances, is dataset dependent.

While our method is intended for tasks with very high dimensional output
spaces, we however notice that even with relatively small numbers of labels,
its accuracy remains comparable to the baseline, with suitable $m$.

To complete the analysis, Appendix~F considers the same experiments with a
different base-learner (Extra Trees of \cite{DBLP:journals/ml/GeurtsEW06}),
showing very similar trends.

\subsection{Input vs output space randomization}
\label{sec:input-output-bias-variance-tradeoff}

We study in this section the interaction of the additional randomization of
the output space with that concerning the input space already built in the Random Forest method.

To this end, we consider the ``Drug-interaction'' dataset ($p=660$ input
features and $d=1554$ output
labels~\cite{yamanishi2011extracting}), and we study the
effect of parameter $k$ controlling the input space randomization of the Random
Forest method with the randomization of the output space by Gaussian
projections controlled by the parameter $m$. To this end, Figure
\ref{fig:max-features-drug-interaction} shows the evolution of the accuracy for
growing values of $k$ (i.e. decreasing strength of the input space
randomization), for three different quite low values of $m$ (in this case $m
\in \left\{1, \ln{d}, 2 \ln{d}\right\}$).  We observe that Random Forests
learned on a very low-dimensional Gaussian subspace (red, blue and pink curves)
yield essentially better performances than Random Forests on the original
output space, and also that their behaviour with respect to the parameter $k$
is quite different. On this dataset, the output-space randomisation makes the
method completely immune to the `over-fitting' phenomenon observed for high
values of $k$ with the baseline method (green curve).

\begin{figure}[htb]
\centering
\includegraphics[width=0.45\textwidth]{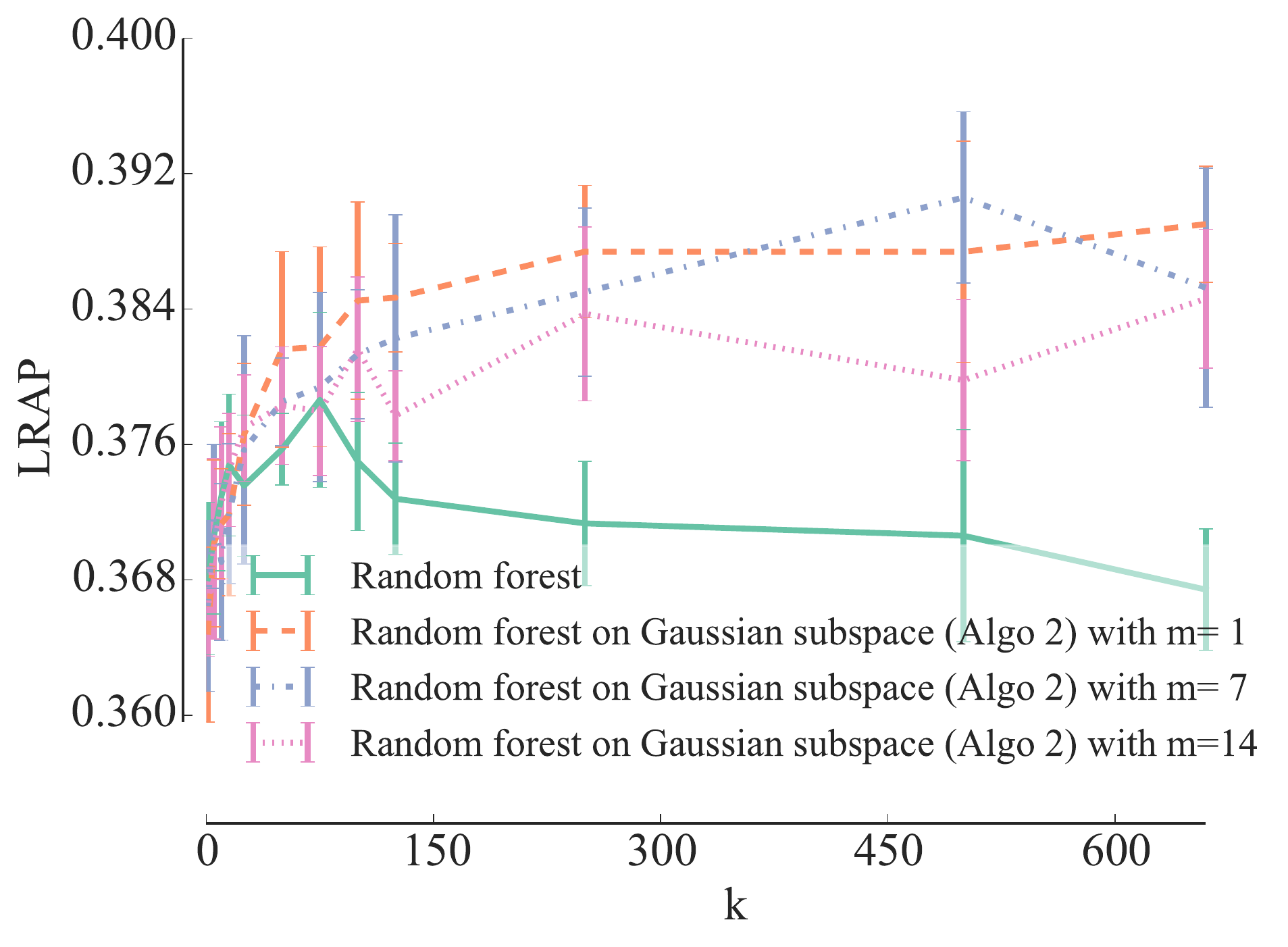}\vspace*{-3mm}
\caption{Output randomization with Gaussian projections yield better average
         precision than the original output space on the ``Drug-Interaction''
         dataset ($n_{\min}=1$ , $t=100$).}
\label{fig:max-features-drug-interaction}
\end{figure}

We refer the reader to a similar study on the ``Delicious'' dataset given in
the Appendix~E (supplementary material), which shows that the interaction
between $m$ and $k$ may be different from one dataset to another.
It is thus advisable to jointly optimize the value of $m$ and $k$, so as to
maximise the tradeoff between accuracy and computing times in a problem and
algorithm specific way.

\subsection{Alternative output dimension reduction techniques}
\label{sec:alternative-output-transformation}

In this section, we study Algorithm~\ref{alg:output-subspace-tree-ensemble}
when it is combined with alternative output-space dimensionality reduction
techniques. We focus again on the ``Delicious'' dataset,
but similar trends could be observed on other datasets.

Figure~\ref{fig:m_delicious_gaussian_rademacher_hadamard} first compares
Gaussian random projections with two other dense projections: Rademacher
matrices with $s=1$ (cf. Section~2.2) and compression matrices obtained by
sub-sampling (without replacement) Hadamard
matrices~\cite{candes2011probabilistic}.  We observe that Rademacher and
subsample-Hadamard sub-spaces behave very similarly to Gaussian random
projections.

In a second step, we compare Gaussian random projections with two
(very) sparse projections: first, sparse Rademacher sub-spaces
obtained by setting the sparsity parameter $s$ to $3$ and $\sqrt{d}$,
selecting respectively about 33\% and 2\% of the original outputs to
compute each component, and second, sub-sampled identity subspaces,
similar to \cite{tsoumakas2007random}, where each of the $m$ selected
components corresponds to a randomly chosen original label and also
preserve sparsity. Sparse projections are very interesting from a
computational point of view as they require much less operations to
compute the projections but the number of components required for
condition (\ref{eqn:js}) to be satisfied is typically higher than for
dense projections \cite{DBLP:conf/kdd/LiHC06,candes2011probabilistic}. Figure~\ref{fig:delicious_vs_n_proj_unitary}
compares these three projection methods with standard Random Forests
on the ``delicious'' dataset. All three projection methods converge to
plain Random Forests as the number of components $m$ increases but
their behaviour at low $m$ values are very different. Rademacher
projections converge faster with $s=3$ than with $s=1$ and
interestingly, the sparsest variant ($s=\sqrt{d}$) has its optimum at
$m=1$ and improves in this case over the Random Forests
baseline. Random output subspaces converge slower but they lead to a
notable improvement of the score over baseline Random Forests. This
suggests that although their theoretical guarantees are less good,
sparse projections actually provide on this problem a better
bias/variance tradeoff than dense ones when used in the context of
Algorithm 2.

Another popular dimension reduction technique is the principal component
analysis (PCA). In Figure~\ref{fig:delicious_vs_n_proj_pca}, we repeat the same
experiment to compare PCA with Gaussian random projections. Concerning PCA, the
curve is generated in decreasing order of eigenvalues, according to their
contribution to the explanation of the output-space variance. We observe that
this way of doing is far less effective than the random projection techniques
studied previously.

\begin{figure}[ht]
\centering
\subfigure[Computing the impurity criterion on a dense Rademacher or on a
           subsample-Hadamard output sub-space is another efficient way to learn
           tree ensembles.]{
              \includegraphics[width=0.45\textwidth]{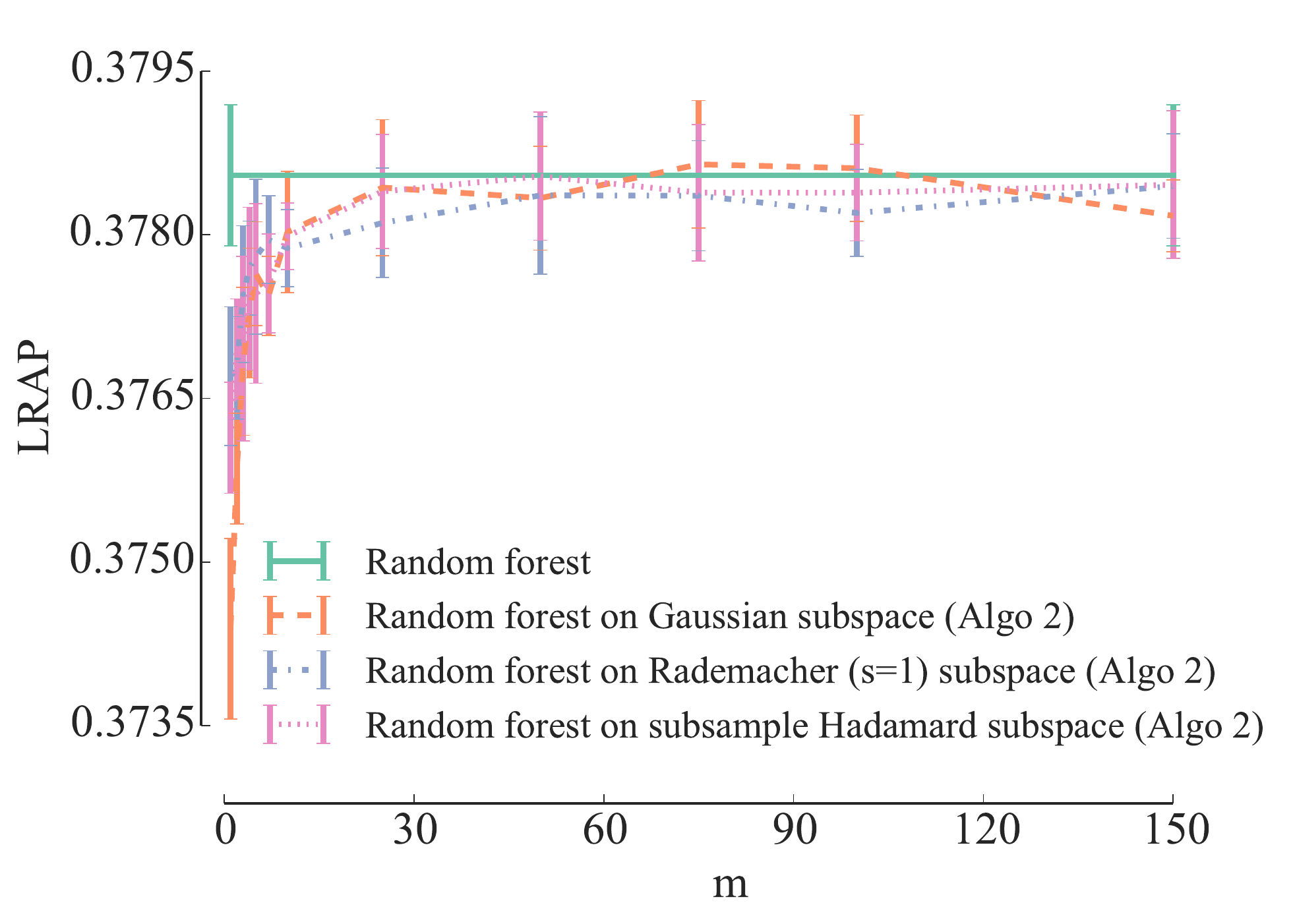}
              \label{fig:m_delicious_gaussian_rademacher_hadamard}
          } \qquad
\subfigure[Sparse random projections output sub-space yield better average precision
           than on the original output space.]{
              \includegraphics[width=0.45\textwidth]{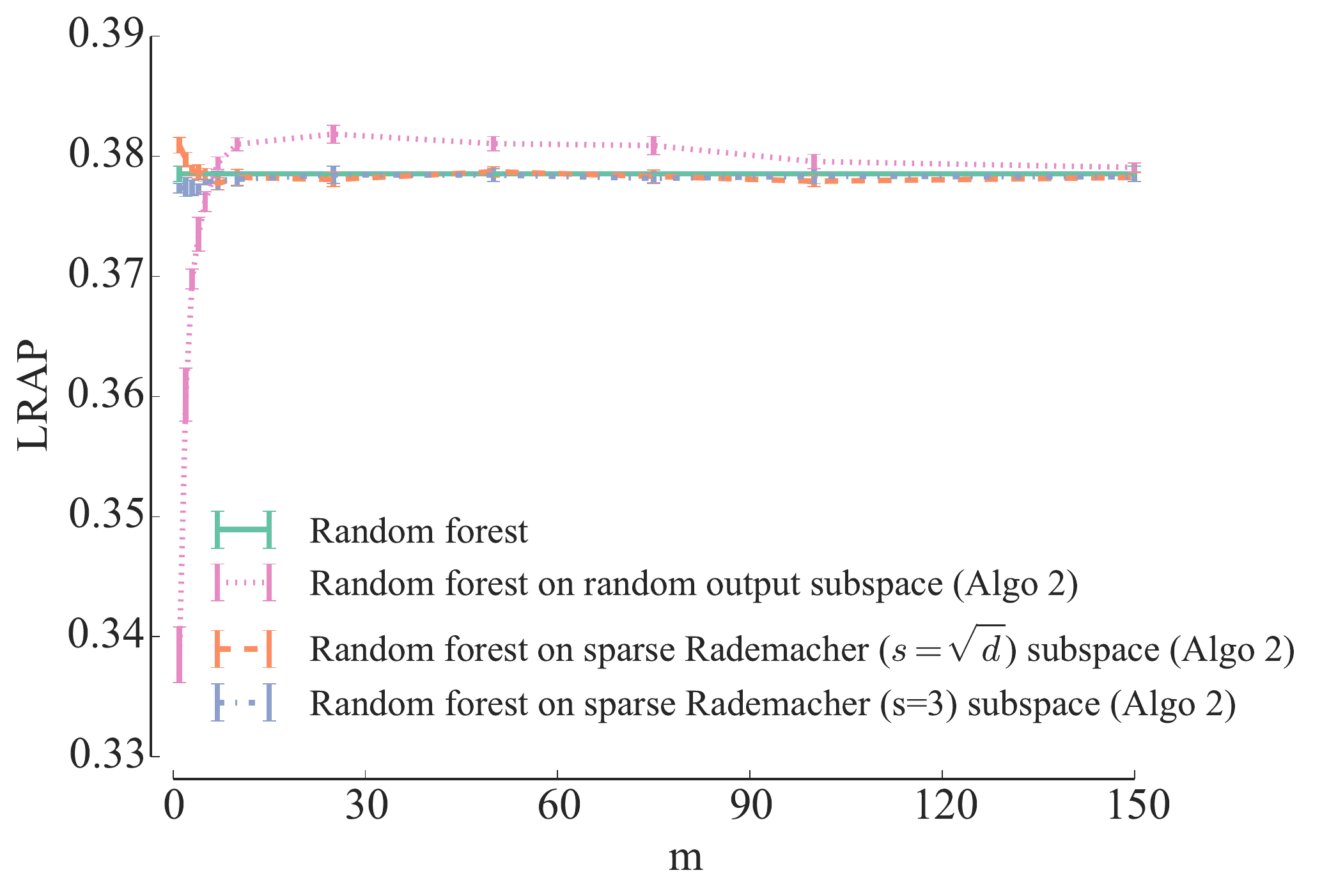}
              \label{fig:delicious_vs_n_proj_unitary}
          } \\
\subfigure[PCA compared with Gaussian subspaces.]{
              \includegraphics[width=0.45\textwidth]{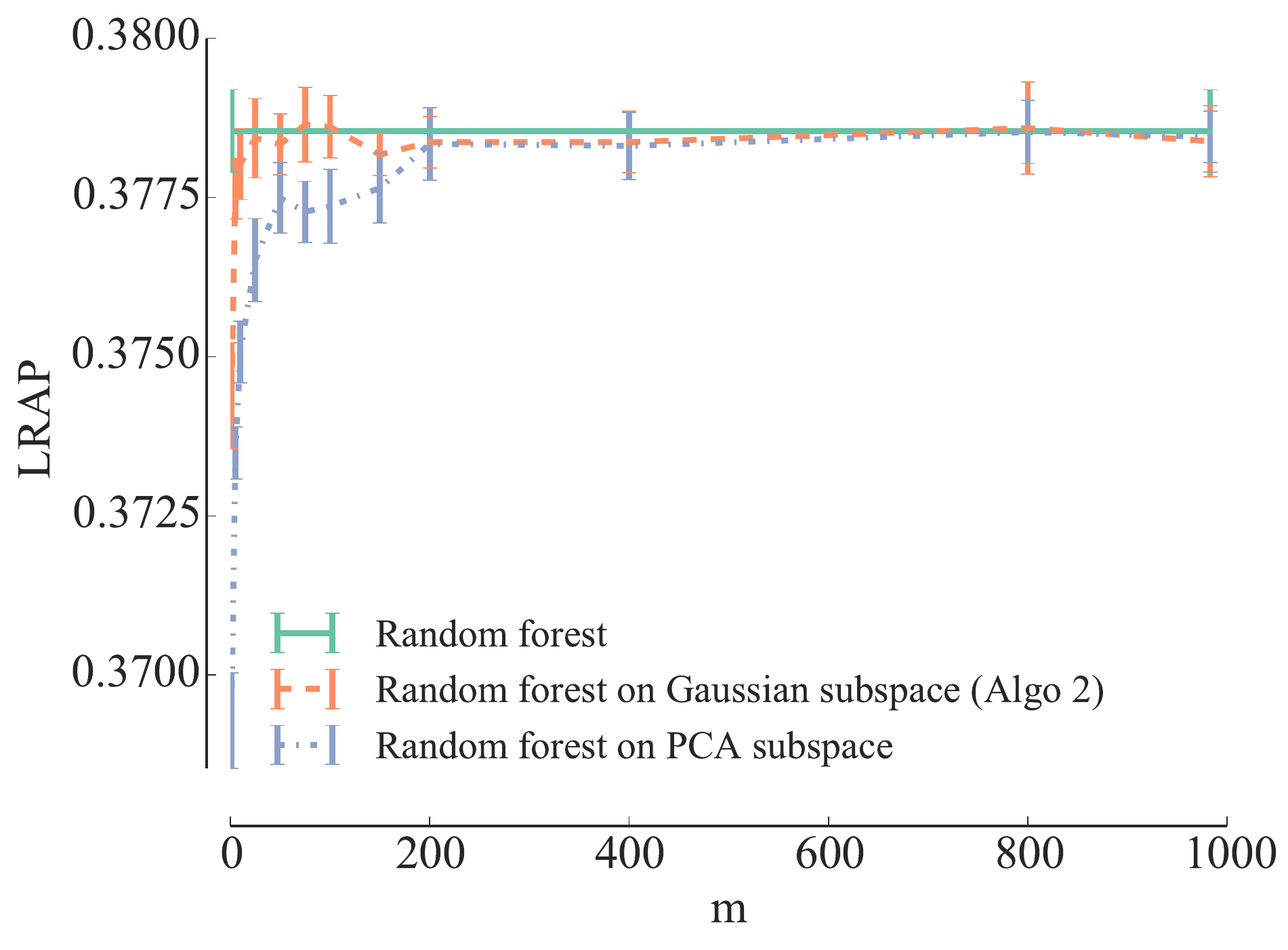}
              \label{fig:delicious_vs_n_proj_pca}
          }

\caption{``Delicious'' dataset, $t=100$, $k=\sqrt{p}$, $n_{\min}=1$.}
\end{figure}

\subsection{Learning stage computing times}\label{sec:benchmarking}

Our implementation of the learning algorithms  is based on the
\textit{scikit-learn} Python package version 0.14-dev
\cite{pedregosa2011scikit}. To fix ideas about computing times, we report
these obtained on a Mac Pro 4.1 with a dual Quad-Core Intel Xeon processor at
2.26 GHz, on the ``Delicious'' dataset.
Matrix operation, such as random projections, are performed with the
BLAS and the LAPACK from the Mac OS X \textit{Accelerate} framework. Reported
times are obtained by summing the user and sys time of the \textit{time} UNIX
utility.

The reported timings correspond to the following operation: (i) load
the dataset in memory, (ii) execute the algorithm. All methods use the
same code to build trees. In these conditions, learning a random
forest on the original output space ($t=100$, $n_{\min}=1$,
$k=\sqrt{d}$) takes 3348~ s; learning the same model on a Gaussian
output space of size $m=25$ requires 311~s, while $m=1$ and $m=250$
take respectively 236~s and 1088~s. Generating a
Gaussian sub-space of size $m=25$ and projecting the output data of
the training samples is done in less than 0.25~s, while $m=1$ and
$m=250$ takes around 0.07~s and 1~s respectively. The time needed to
compute the projections is thus negligible with respect to the time
needed for the tree construction.

We see that a speed-up of an order of magnitude could be obtained,
while at the same time preserving accuracy with respect to the
baseline Random Forests method. Equivalently, for a
fixed computing time budget, randomly projecting the output space
allows to build more trees and thus to improve predictive
performances with respect to standard Random Forests.

\section{Conclusions} \label{sec:conclusions}

This paper explores the use of random output space projections combined with
tree-based ensemble methods to address large-scale multi-label classification
problems. We study two algorithmic variants that either build a tree-based
ensemble model on a single shared random subspace or build each tree in the
ensemble on a newly drawn random subspace. The second approach is shown
theoretically and empirically to always outperform the first in terms of
accuracy. Experiments on 24 datasets show that on most problems, using gaussian
projections allows to reduce very drastically the size of the output space, and
therefore computing times, without affecting accuracy. Remarkably, we also show
that by adjusting jointly the level of input and output randomizations and
choosing appropriately the projection method, one could also improve predictive
performance over the standard Random Forests, while still improving very
significantly computing times. As future work, it would be very interesting to
propose efficient techniques to automatically adjust these parameters, so as to
reach the best tradeoff between accuracy and computing times on a given
problem.

To best of our knowledge, our work is the first to study random output
projections in the context of multi-output tree-based ensemble methods. The
possibility with these methods to relabel tree leaves with predictions in the
original output space makes this combination very attractive. Indeed, unlike
similar works with linear models \cite{DBLP:conf/nips/HsuKLZ09,NIPS2013_5083},
our approach only relies on Johnson-Lindenstrauss lemma, and not on any output
sparsity assumption, and also does not require to use any output reconstruction
method. Besides multi-label classification, we would like to test our method on
other, not necessarily sparse, multi-output prediction problems.

\paragraph{Acknowledgements.} Arnaud Joly is research fellow of the FNRS,
Belgium. This work is supported by PASCAL2 and the IUAP DYSCO, initiated by
the Belgian State, Science Policy Office.


\bibliographystyle{splncs03}
\bibliography{references}

\begin{thebibliography}{10}
\providecommand{\url}[1]{\texttt{#1}}
\providecommand{\urlprefix}{URL }

\bibitem{DBLP:journals/jcss/Achlioptas03}
Achlioptas, D.: Database-friendly random projections: Johnson-lindenstrauss
  with binary coins. J. Comput. Syst. Sci.  66(4),  671--687 (2003)

\bibitem{DBLP:conf/www/AgrawalGPV13}
Agrawal, R., Gupta, A., Prabhu, Y., Varma, M.: Multi-label learning with
  millions of labels: recommending advertiser bid phrases for web pages. In:
  Proceedings of the 22nd international conference on World Wide Web. pp.
  13--24. International World Wide Web Conferences Steering Committee (2013)

\bibitem{Barutcuoglu06:jrnl}
Barutcuoglu, Z., Schapire, R.E., Troyanskaya, O.G.: Hierarchical multi-label
  prediction of gene function. Bioinformatics  22(7),  830--836 (2006)

\bibitem{blockeel-1998}
Blockeel, H., De~Raedt, L., Ramon, J.: Top-down induction of clustering trees.
  In: Proceedings of ICML 1998. pp. 55--63 (1998)

\bibitem{boutell2004learning}
Boutell, M.R., Luo, J., Shen, X., Brown, C.M.: Learning multi-label scene
  classification. Pattern recognition  37(9),  1757--1771 (2004)

\bibitem{DBLP:journals/ml/Breiman01}
Breiman, L.: Random forests. Machine Learning  45(1),  5--32 (2001)

\bibitem{DBLP:books/wa/BreimanFOS84}
Breiman, L., Friedman, J.H., Olshen, R.A., Stone, C.J.: Classification and
  Regression Trees. Wadsworth (1984)

\bibitem{candes2011probabilistic}
Candes, E.J., Plan, Y.: A probabilistic and ripless theory of compressed
  sensing. Information Theory, IEEE Transactions on  57(11),  7235--7254 (2011)

\bibitem{DBLP:conf/icml/DembczynskiCH10}
Cheng, W., H{\"u}llermeier, E., Dembczynski, K.J.: Bayes optimal multilabel
  classification via probabilistic classifier chains. In: Proceedings of the
  27th international conference on machine learning (ICML-10). pp. 279--286
  (2010)

\bibitem{NIPS2013_5083}
Cisse, M.M., Usunier, N., Arti\`{e}res, T., Gallinari, P.: Robust bloom filters
  for large multilabel classification tasks. In: Burges, C., Bottou, L.,
  Welling, M., Ghahramani, Z., Weinberger, K. (eds.) Advances in Neural
  Information Processing Systems 26, pp. 1851--1859 (2013)

\bibitem{Clare03:phd}
Clare, A.: Machine learning and data mining for yeast functional genomics.
  Ph.D. thesis, University of Wales Aberystwyth, Aberystwyth, Wales, UK (2003)

\bibitem{dekel2010multiclass}
Dekel, O., Shamir, O.: Multiclass-multilabel classification with more classes
  than examples. In: International Conference on Artificial Intelligence and
  Statistics. pp. 137--144 (2010)

\bibitem{Dimitrovski11:jrnl}
Dimitrovski, I., Kocev, D., Loskovska, S., D{\v{z}}eroski, S.: Hierarchical
  classification of diatom images using ensembles of predictive clustering
  trees. Ecological Informatics  7(1),  19--29 (2012)

\bibitem{diplaris2005protein}
Diplaris, S., Tsoumakas, G., Mitkas, P.A., Vlahavas, I.: Protein classification
  with multiple algorithms. In: Advances in Informatics, pp. 448--456. Springer
  (2005)

\bibitem{DBLP:conf/eccv/DuyguluBFF02}
Duygulu, P., Barnard, K., de~Freitas, J.F.G., Forsyth, D.A.: Object recognition
  as machine translation: Learning a lexicon for a fixed image vocabulary. In:
  Computer Vision—ECCV 2002, pp. 97--112. Springer (2006)

\bibitem{elisseeff2001kernel}
Elisseeff, A., Weston, J.: A kernel method for multi-labelled classification.
  In: Advances in neural information processing systems. pp. 681--687 (2001)

\bibitem{DBLP:journals/ml/GeurtsEW06}
Geurts, P., Ernst, D., Wehenkel, L.: Extremely randomized trees. Machine
  Learning  63(1),  3--42 (2006)

\bibitem{DBLP:conf/icml/GeurtsWd06}
Geurts, P., Wehenkel, L., d'Alch{\'e} Buc, F.: Kernelizing the output of
  tree-based methods. In: Proceedings of the 23rd international conference on
  Machine learning. pp. 345--352. Acm (2006)

\bibitem{DBLP:conf/nips/HsuKLZ09}
Hsu, D., Kakade, S., Langford, J., Zhang, T.: Multi-label prediction via
  compressed sensing. In: Bengio, Y., Schuurmans, D., Lafferty, J., Williams,
  C.K.I., Culotta, A. (eds.) Advances in Neural Information Processing Systems
  22, pp. 772--780 (2009)

\bibitem{johnson1984extensions}
Johnson, W.B., Lindenstrauss, J.: Extensions of {Lipschitz} mappings into a
  {Hilbert} space. Contemporary mathematics  26(189-206), ~1 (1984)

\bibitem{katakis2008multilabel}
Katakis, I., Tsoumakas, G., Vlahavas, I.: Multilabel text classification for
  automated tag suggestion. In: Proceedings of the ECML/PKDD (2008)

\bibitem{klimt2004enron}
Klimt, B., Yang, Y.: The enron corpus: A new dataset for email classification
  research. In: Machine learning: ECML 2004, pp. 217--226. Springer (2004)

\bibitem{DBLP:journals/pr/KocevVSD13}
Kocev, D., Vens, C., Struyf, J., Dzeroski, S.: Tree ensembles for predicting
  structured outputs. Pattern Recognition  46(3),  817--833 (2013)

\bibitem{DBLP:conf/kdd/LiHC06}
Li, P., Hastie, T.J., Church, K.W.: Very sparse random projections. In:
  Proceedings of the 12th ACM SIGKDD international conference on Knowledge
  discovery and data mining. pp. 287--296. ACM (2006)

\bibitem{DBLP:journals/pr/MadjarovKGD12}
Madjarov, G., Kocev, D., Gjorgjevikj, D., Dzeroski, S.: An extensive
  experimental comparison of methods for multi-label learning. Pattern
  Recognition  45(9),  3084--3104 (2012)

\bibitem{DBLP:conf/lrec/MenciaF10}
Menc{\'\i}a, E.L., F{\"u}rnkranz, J.: Efficient multilabel classification
  algorithms for large-scale problems in the legal domain. In: Semantic
  Processing of Legal Texts, pp. 192--215. Springer (2010)

\bibitem{pedregosa2011scikit}
Pedregosa, F., Varoquaux, G., Gramfort, A., Michel, V., Thirion, B., Grisel,
  O., Blondel, M., Prettenhofer, P., Weiss, R., Dubourg, V., et~al.:
  Scikit-learn: Machine learning in python. The Journal of Machine Learning
  Research  12,  2825--2830 (2011)

\bibitem{read2009}
Read, J., Pfahringer, B., Holmes, G., Frank, E.: Classifier chains for
  multi-label classification. In: Buntine, W., Grobelnik, M., Mladenić, D.,
  Shawe-Taylor, J. (eds.) Machine Learning and Knowledge Discovery in
  Databases, Lecture Notes in Computer Science, vol. 5782, pp. 254--269.
  Springer Berlin Heidelberg (2009)

\bibitem{rousu2005learning}
Rousu, J., Saunders, C., Szedmak, S., Shawe-Taylor, J.: Learning hierarchical
  multi-category text classification models. In: Proceedings of the 22nd
  international conference on Machine learning. pp. 744--751. ACM (2005)

\bibitem{snoek2006challenge}
Snoek, C.G.M., Worring, M., Van~Gemert, J.C., Geusebroek, J.M., Smeulders,
  A.W.M.: The challenge problem for automated detection of 101 semantic
  concepts in multimedia. In: Proceedings of the 14th annual ACM international
  conference on Multimedia. pp. 421--430. ACM (2006)

\bibitem{srivastava2005discovering}
Srivastava, A.N., Zane-Ulman, B.: Discovering recurring anomalies in text
  reports regarding complex space systems. In: Aerospace Conference, 2005 IEEE.
  pp. 3853--3862. IEEE (2005)

\bibitem{tsoumakas2008effective}
Tsoumakas, G., Katakis, I., Vlahavas, I.: Effective and efficient multilabel
  classification in domains with large number of labels. In: Proc. ECML/PKDD
  2008 Workshop on Mining Multidimensional Data (MMD’08). pp. 30--44 (2008)

\bibitem{DBLP:journals/tkde/TsoumakasKV11}
Tsoumakas, G., Katakis, I., Vlahavas, I.P.: Random k-labelsets for multilabel
  classification. IEEE Trans. Knowl. Data Eng.  23(7),  1079--1089 (2011)

\bibitem{tsoumakas2007random}
Tsoumakas, G., Vlahavas, I.: Random k-labelsets: An ensemble method for
  multilabel classification. In: Machine Learning: ECML 2007, pp. 406--417.
  Springer (2007)

\bibitem{tsoumakas2008multi}
Tsoumakas, K.T.G., Kalliris, G., Vlahavas, I.: Multi-label classification of
  music into emotions. In: ISMIR 2008: Proceedings of the 9th International
  Conference of Music Information Retrieval. p. 325. Lulu. com (2008)

\bibitem{turnbull2008semantic}
Turnbull, D., Barrington, L., Torres, D., Lanckriet, G.: Semantic annotation
  and retrieval of music and sound effects. Audio, Speech, and Language
  Processing, IEEE Transactions on  16(2),  467--476 (2008)

\bibitem{vens2008decision}
Vens, C., Struyf, J., Schietgat, L., D{\v{z}}eroski, S., Blockeel, H.: Decision
  trees for hierarchical multi-label classification. Machine Learning  73(2),
  185--214 (2008)

\bibitem{yamanishi2011extracting}
Yamanishi, Y., Pauwels, E., Saigo, H., Stoven, V.: Extracting sets of chemical
  substructures and protein domains governing drug-target interactions. Journal
  of chemical information and modeling  51(5),  1183--1194 (2011)

\bibitem{zhou2012multi}
Zhou, T., Tao, D.: Multi-label subspace ensemble. In: International Conference
  on Artificial Intelligence and Statistics. pp. 1444--1452 (2012)

\end{thebibliography}

\clearpage

\renewcommand{\theequation}{a.\arabic{equation}}
\setcounter{equation}{0}

\appendix

\section{Proof of equation (3)}
\label{sec:proof-var-distance}

The sum of the variances of $n$ observations drawn from a random vector
$y \in \mathbb{R}^d$ can be interpreted as a sum of squared euclidean distances
between the pairs of observations
\begin{equation}
\Var((y^i)_{i=1}^n) = \frac{1}{2n^2} \sum_{i=1}^n \sum_{j=1}^n ||y^i - y^j||^2. \label{sum-of-dist}
\end{equation}

\begin{proof}\small
\begin{align}
& \Var((y^i)_{i=1}^n) \nonumber \\
&\stackrel{\mbox{\tiny def}}{=} \frac{1}{n} \sum_{i=1}^n ||y^i - \frac{1}{n} \sum_{j=1}^n y^j||^2 \\
&= \frac{1}{n} \sum_{i=1}^n
      (y^i - \frac{1}{n} \sum_{j=1}^n y^j)^T
      (y^i - \frac{1}{n} \sum_{k=1}^n y^k) \\
&= \frac{1}{n} \sum_{i=1}^n
      \left(
      {y^i}^T y^i -
      \frac{2}{n}  \sum_{j=1}^n {y^i}^T y^j +
      \frac{1}{n^2} \sum_{j=1}^n \sum_{k=1}^n {y^j}^T y^k
      \right)\\
&=  \frac{1}{n} \sum_{i=1}^n {y^i}^T y^i
    - \frac{2}{n^2} \sum_{i=1}^n \sum_{j=1}^n {y^i}^T y^j
    + \frac{1}{n^2} \sum_{j=1}^n \sum_{k=1}^n {y^j}^T y^k \\
&= \frac{1}{n} \sum_{i=1}^n {y^i}^T y^i
   - \frac{1}{n^2} \sum_{i=1}^n \sum_{j=1}^n {y^i}^T y^j \\
&=   \frac{1}{2n} \sum_{i=1}^n {y^i}^T y^i
   + \frac{1}{2n} \sum_{j=1}^n {y^j}^T y^j
   - \frac{1}{n^2} \sum_{i=1}^n \sum_{j=1}^n {y^i}^T y^j \\
&=  \frac{1}{2 n^2} \sum_{i=1}^n \sum_{j=1}^n {y^i}^T\hspace*{-1mm} y^i
   \hspace*{-1mm}+ \hspace*{-1mm}\frac{1}{2 n^2} \sum_{i=1}^n \sum_{j=1}^n {y^j}^T\hspace*{-1mm} y^j
   \hspace*{-1mm}- \hspace*{-1mm}\frac{1}{n^2} \sum_{i=1}^n \sum_{j=1}^n {y^i}^T\hspace*{-1mm} y^j \\
&= \frac{1}{2n^2} \sum_{i=1}^n \sum_{j=1}^n
      \left(
      {y^i}^T y^i + {y^j}^T y^j - 2 {y^i}^T y^j
      \right) \\
&= \frac{1}{2n^2} \sum_{i=1}^n \sum_{j=1}^n ||y^i - y^j||^2.
\end{align}
\end{proof}

\section{Proof of Theorem 1}
\label{sec:proof-thm-var-jl-lemma}

{\bf Theorem 1} {\em
Given $\epsilon > 0$, a sample $(y^i)_{i=1}^{n}$ of $n$ points $y \in
\mathbb{R}^d$ and a projection matrix $\Phi\in\mathbb{R}^{m \times d}$ such
that for all $i, j  \in \{1, \ldots , n\}$ the Johnson-Lindenstrauss lemma
holds, we have also:
\begin{align}
(1 - \epsilon) \Var((y^i)_{i=1}^n) & \leq \Var((\Phi y^i))_{i=1}^n) \nonumber \\
                                   & \leq (1 + \epsilon) \Var((y^i)_{i=1}^n). \label{var-ineq}
\end{align}}

\begin{proof}\small
This result directly follows from the Johnson-Lindenstrauss Lemma and from eqn. (\ref{sum-of-dist}).

From the Johnson-Lindenstrauss Lemma we have for any $i,j$
\begin{equation}
(1 - \epsilon) ||y^i - y^j||^2 \leq || \Phi y^i - \Phi y^j ||^2
                               \leq (1 + \epsilon) || y^i - y^j||^2. \label{ineq-vect}
\end{equation}
By summing the three terms of (\ref{ineq-vect}) over all pairs $i,j$
and dividing by $1 / (2n^2)$ and by then using eqn. (\ref{sum-of-dist}), we get eqn. (\ref{var-ineq}).\end{proof}

\section{Bias/variance analysis}\label{app:biasvar}

In this subsection, we adapt the bias/variance analysis of randomised supervised learning algorithms carried out in
\cite{DBLP:journals/ml/GeurtsEW06},  to assess the effect of random output
projections in the context of the two algorithms studied in our paper.

\subsection{Single random trees.}
Let us denote by $f(.;ls,\phi,\epsilon):{\cal X}\rightarrow \mathbb{R}^d$ a single
multi-output (random) tree obtained from a projection matrix $\phi$ (below we use $\Phi$ to denote the corresponding random variable), where
$\epsilon$ is the value of a random variable $\varepsilon$ capturing the random
perturbation scheme used to build this tree (e.g., bootstrapping and/or random
input space selection). Denoting by $Err(f(x;ls,\phi,\epsilon))$ the square
error of this model at some point $x\in{\cal X}$ defined by:
$$E_{Y|x}\{||Y-f(x;ls,\phi,\epsilon\})||^2\}.$$

{\small

The average of this square error can decomposed as follows:
\begin{eqnarray*}\small
&&E_{LS,\Phi,\varepsilon}\{Err(f(x;LS,\Phi,\varepsilon))\}\\
&=&\sigma^2_R(x)+||f_B(x)-\bar{f}(x)||^2+\Var_{LS,\Phi,\varepsilon}\{f(x;LS,\Phi,\varepsilon)\},
\end{eqnarray*}
where $$\bar{f}(x)\stackrel{\text{def}}{=}E_{LS,\Phi,\varepsilon}\{f(x;LS,\Phi,\varepsilon)\},f_B(x)=E_{Y|x}\{Y\},$$
and
$$\small \Var_{LS,\Phi,\varepsilon}\{f(x;LS,\Phi,\varepsilon)\} \stackrel{\text{def}}{=} E_{LS,\Phi,\varepsilon}\{||f(x;LS,\Phi,\varepsilon)-\bar{f}(x)||^2\}$$

The three terms of this decomposition are respectively the residual error,
the bias, and the variance of this estimator (at $x$).

The variance term can be further decomposed as follows using the law of total
variance:
\begin{eqnarray}\small
\Var_{LS,\Phi,\varepsilon}\{f(x;LS,\Phi,\varepsilon)\}\hspace*{-4cm} &\hspace{1cm}&\notag\\
&=&\Var_{LS}\{E_{\Phi,\varepsilon|LS}\{f(x;LS,\Phi,\varepsilon)\}\}\notag\\
&&+E_{LS}\{\Var_{\Phi,\varepsilon|LS}\{f(x;LS,\Phi,\varepsilon)\}\}.\label{eqn:totalvar1}
\end{eqnarray}
The first term is the variance due to the learning sample randomization and the
second term is the average variance (over $LS$) due to both the random forest
randomization and the random output projection. By using the law of total
variance a second time, the second term of (\ref{eqn:totalvar1}) can be further
decomposed as follows:
\begin{align}\small
&E_{LS}\{\Var_{\Phi,\varepsilon|LS}\{f(x;LS,\Phi,\varepsilon)\}\} \notag\\
&= E_{LS}\{\Var_{\Phi|LS}\{E_{\varepsilon|LS,\Phi}\{f(x;LS,\Phi,\varepsilon)\}\}\}\notag\\
&\quad + E_{LS}\{E_{\Phi|LS}\{\Var_{\varepsilon|LS,\Phi}\{f(x;LS,\Phi,\varepsilon)\}\}\}.
\label{eqn:totalvar2}
\end{align}
The first term of this decomposition is the variance due to the random
choice of a projection and the second term is the average variance due
to the random forest randomization.  Note that all these terms are non
negative. In what follows, we will denote these three terms respectively $V_{LS}(x)$, $V_{Algo}(x)$, and $V_{proj}(x)$. We thus have:
$$\Var_{LS,\Phi,\varepsilon}\{f(x;LS,\Phi,\varepsilon)\}=V_{LS}(x)+V_{Algo}(x)+V_{Proj}(x),$$
with
\begin{eqnarray*}
V_{LS}(x)&=&\Var_{LS}\{E_{\Phi,\varepsilon|LS}\{f(x;LS,\Phi,\varepsilon)\}\}\\
V_{Algo}(x)&=&E_{LS}\{E_{\Phi|LS}\{\Var_{\varepsilon|LS,\Phi}\{f(x;LS,\Phi,\varepsilon)\}\}\},\\
V_{Proj}(x)&=&E_{LS}\{\Var_{\Phi|LS}\{E_{\varepsilon|LS,\Phi}\{f(x;LS,\Phi,\varepsilon)\}\}\},
\end{eqnarray*}

}

\subsection{Ensembles of $t$ random trees.}

When the random projection is fixed for all $t$ trees in the ensemble (Algorithm~1), the
algorithm computes an approximation, denoted $f_1(x;ls,\phi, \epsilon^t)$, that
takes the following form:
$$f_1(x;ls,\phi, \epsilon^t)=\frac{1}{t}\sum_{i=1}^t f(x;ls,\phi, \epsilon_i),$$
where $\epsilon^t=(\epsilon_1,\ldots,\epsilon_{t})$ a vector of
i.i.d. values of the random variable $\varepsilon$. When a different random projection
is chosen for each tree (Algorithm~2),
the algorithm computes an approximation, denoted by
$f_2(x;ls,\phi^t,\epsilon^t)$, of the following form:
$$f_2(x;ls,\phi^t,\epsilon^t)=\frac{1}{t}\sum_{i=1}^t f(x;ls,\phi_i,\epsilon_i),$$ where
$\phi^t=(\phi_1,\ldots,\phi_t)$ is also a vector of i.i.d. random projection matrices.

We would like to compare the average errors of these two algorithms with the
average errors of the original single tree method, where the average is taken
for all algorithms over their random parameters (that include the learning
sample).

{\small

Given that all trees are grown independently of each other, one can show that
the average models corresponding to each algorithm are equal:
\begin{eqnarray*}
\bar{f}(x)&=&E_{LS,\Phi,\varepsilon^t}\{f_1(x;LS,\Phi,\varepsilon^t)\}\\
&=&E_{LS,\Phi^t,\varepsilon^t}\{f_2(x;LS,\Phi^t,\varepsilon^t)\}.
\end{eqnarray*}
They thus all have the exact same bias (and residual error) and differ only in
their variance.

Using the same argument, the first term of the variance decomposition
in (\ref{eqn:totalvar1}), ie. $V_{LS}(x)$, is the same for all three
algorithms since:
\begin{eqnarray*}
&&E_{\Phi,\varepsilon|LS}\{f(x;LS,\Phi,\varepsilon)\}\\
&=&E_{\Phi,\varepsilon^t|LS}\{f_1(x;LS,\Phi,\varepsilon^t)\}\\
&=&E_{\Phi^t,\varepsilon^t|LS}\{f_2(x;LS,\Phi^t,\varepsilon^t)\}.
\end{eqnarray*}
Their variance thus only differ in the second term in (\ref{eqn:totalvar1}).

Again, because of the conditional independence of the ensemble terms given the $ls$ and projection matrix $\phi$, Algorithm 1,
which keeps the output projection fixed for all trees, is such that
$$E_{\epsilon^t|LS,\Phi}\{f_1(x;LS,\Phi,\varepsilon^t)\}=E_{\epsilon|LS,\Phi}\{f(x;LS,\Phi,\varepsilon)\}$$
and
\[
\Var_{\varepsilon^t|LS,\Phi}\{f_1(x;LS,\Phi,\varepsilon^t)\} =
\frac{1}{t}\Var_{\varepsilon|LS,\Phi}\{f(x;LS,\Phi,\varepsilon)\}.
\]

It thus divides the second term of (\ref{eqn:totalvar2}) by the number $t$ of ensemble terms.
Algorithm 2, on the other hand, is such that:
\[
\Var_{\Phi^t,\varepsilon^t|LS}\{f_2(x;LS,\Phi,\varepsilon^t)\}
= \frac{1}{t}\Var_{\Phi,\varepsilon|LS}\{f(x;LS,\Phi,\varepsilon)\},
\]

and thus divides the second term of (\ref{eqn:totalvar1}) by $t$.

Putting all these results together one gets that:
\begin{eqnarray*}
&&\hspace*{-5mm}E_{LS,\Phi,\varepsilon}\{Err(f(x;LS,\Phi,\varepsilon^t))\}\\
&&\hspace{-0.2cm}=\sigma^2_R(x)+B^2(x)+V_{LS}(x)+V_{Algo}(x)+V_{Proj}(x),\\
&&\hspace*{-5mm}E_{LS,\Phi,\varepsilon^t}\{Err(f_1(x;LS,\Phi,\varepsilon^t))\}\\
&&\hspace{-0.2cm}=\sigma^2_R(x)+B^2(x)+V_{LS}(x)+\frac{V_{Algo}(x)}{t}+V_{Proj}(x),\\
&&\hspace*{-5mm}E_{LS,\Phi^t,\varepsilon^t}\{Err(f_2(x;LS,\Phi^t,\varepsilon^t))\}\\
&&\hspace{-0.2cm} =\sigma^2_R(x)+B^2(x)+V_{LS}(x)+\frac{V_{Algo}(x)+V_{Proj}(x)}{t}.
\end{eqnarray*}

Given that all terms are positive, this result clearly shows that Algorithm~2 can not be worse than Algorithm~1.

}

\section{Description of the datasets}

Experiments are performed on several multi-label datasets:
the yeast~\cite{elisseeff2001kernel}
dataset in the biology domain;
the corel5k~\cite{DBLP:conf/eccv/DuyguluBFF02} and
the scene~\cite{boutell2004learning}
datasets in the image domain;
the emotions~\cite{tsoumakas2008multi} and
the CAL500~\cite{turnbull2008semantic}
datasets in the music domain;
the bibtex~\cite{katakis2008multilabel},
the bookmarks~\cite{katakis2008multilabel},
the delicious~\cite{tsoumakas2008effective},
the enron~\cite{klimt2004enron},
the EUR-Lex (subject matters, directory codes and eurovoc descriptors)~\cite{DBLP:conf/lrec/MenciaF10}
the genbase~\cite{diplaris2005protein},
the medical\footnote{The medical dataset comes from the computational medicine
center's 2007 medical natural language processing challenge
 \url{http://computationalmedicine.org/challenge/previous}.},
the tmc2007~\cite{srivastava2005discovering}
datasets in the text domain
and the mediamill~\cite{snoek2006challenge}
dataset in the video domain.

Several hierarchical classification tasks
are also studied to increase the diversity in the number of label
and treated as multi-label classification task. Each node of the hierarchy
is treated as one label. Nodes of the hierarchy which never occured in the
training or testing set were removed.
The reuters~\cite{rousu2005learning},
WIPO~\cite{rousu2005learning} datasets are from the text domain. The
Diatoms~\cite{Dimitrovski11:jrnl} dataset is from the image domain.
SCOP-GO~\cite{Clare03:phd}, Yeast-GO~\cite{Barutcuoglu06:jrnl}
and Expression-GO~\cite{vens2008decision} are from the biological domain.
Missing values in the Expression-GO dataset were inferred using
the median for continuous features and the most frequent value for categorical
features using the entire dataset.
The inference of a drug-protein interaction
network~\cite{yamanishi2011extracting} is also considered
either using the drugs to infer the interactions with the protein
(drug-interaction), either using the proteins to infer the interactions
with the drugs (protein-interaction).

Those datasets were selected to have a wide range of number of outputs $d$.
Their basic characteristics are summarized at
Table~\ref{tab:dataset_summary}. For more information on a particular dataset,
please see the relevant paper.

\begin{table}[htb]
\caption{Selected datasets have a number of labels $d$  ranging from 6 up to
         3993 in the biology, the text, the image, the video or the music domain.
         Each dataset has $n_{LS}$ training samples,  $n_{TS}$ testing
         samples and $p$ input features.}
\centering
\renewcommand{\tabcolsep}{1.4mm}
\begin{tabular}{@{} l rrrr  @{}}
\toprule
Datasets                      & $n_{LS}$ & $n_{TS}$ &      $p$ &        $d$\\
\midrule
emotions                      &      391 &      202 &       72 &         6 \\
scene                         &     1211 &     1196 &     2407 &         6 \\
yeast                         &     1500 &      917 &      103 &        14 \\
tmc2007                       &    21519 &     7077 &    49060 &        22 \\
genbase                       &      463 &      199 &     1186 &        27 \\
reuters                       &     2500 &     5000 &    19769 &        34 \\
medical                       &      333 &      645 &     1449 &        45 \\
enron                         &     1123 &      579 &     1001 &        53 \\
mediamill                     &    30993 &    12914 &      120 &       101 \\
Yeast-GO                      &     2310 &     1155 &     5930 &       132 \\
bibtex                        &     4880 &     2515 &     1836 &       159 \\
CAL500                        &      376 &      126 &       68 &       174 \\
WIPO                          &     1352 &      358 &    74435 &       188 \\
EUR-Lex \footnotesize{(subject matters)} &    19348 &    10-cv &     5000 &       201 \\
bookmarks                     &    65892 &    21964 &     2150 &       208 \\
diatoms                       &     2065 &     1054 &      371 &      359 \\
corel5k                       &     4500 &      500 &      499 &       374 \\
EUR-Lex \footnotesize{(directory codes)} &    19348 &    10-cv &     5000 &       412 \\
SCOP-GO                       &     6507 &     3336 &     2003 &       465 \\
delicious                     &    12920 &     3185 &      500 &       983 \\
drug-interaction              &     1396 &      466 &      660 &      1554 \\
protein-interaction           &     1165 &      389 &      876 &      1862 \\
Expression-GO                 &     2485 &      551 &     1288 &      2717 \\
EUR-Lex \scriptsize{(eurovoc descriptors)} &    19348 &    10-cv &     5000 &      3993 \\
\bottomrule
\end{tabular}
\label{tab:dataset_summary}
\end{table}

\section{Experiments with Extra trees}

In this section, we carry out experiments combining Gaussian random
projections (with $m\in\{1,\ln{d},d\}$) with the Extra Trees method
of \cite{DBLP:journals/ml/GeurtsEW06} (see Section 2.1 of the paper
for a very brief description of this method). Results on 23
datasets\footnote{Note that results on the ``EUR-lex'' dataset were
  not available at the time of submitting the paper. They will be
  added in the final version of this appendix.} are compiled in
Table~\ref{tab:extra_trees_results_summary}.

Like for Random Forests, we observe that for all 23 datasets taking
$m=d$ leads to a similar average precision to the standard Random
Forests, ie. no difference superior to one standard deviation of the
error. This is already the case with $m=1$ for 12 datasets and with
$m=\ln{d}$ for 4 more datasets. Interestingly, on 3 datasets with
$m=1$ and 3 datasets with $m=\ln{d}$, the increased randomization
brought by the projections actually improves average precision with
respect to standard Random Forests (bold values in
Table~\ref{tab:extra_trees_results_summary}).

\begin{table*}[htb]
\centering
\caption{Experiments with Gaussian projections and Extra Trees
  (($t=100$, $n_{\min}=1$, $k=\sqrt{p}$). See Section 4.3 and Table~1
  in the paper for the protocol. Mean scores in the last three columns
  are underlined if they show a negative difference with respect to the
  standard Random Forests of more than one standard deviation. Bold
  values highlight improvement over standard RF of more than one standard
  deviation.}
\begin{footnotesize}
\begin{tabular}{@{} lll ll ll ll  @{}}
\toprule
Datasets            & \multicolumn{2}{c}{Extra trees}      & \multicolumn{6}{c}{Extra trees on Gaussian sub-space } \\
\cmidrule(r){4-5} \cmidrule(r){6-7}\cmidrule(r){8-9}
                   &          &             & \multicolumn{2}{c}{$m=1$} & \multicolumn{2}{c}{$m\hspace*{-0.9mm}=\hspace*{-0.9mm}\lfloor\hspace*{-0.4mm} 0.5\hspace*{-0.8mm} +\hspace*{-0.8mm} \ln{d}\hspace*{-0.4mm}\rfloor$} & \multicolumn{2}{c}{$m=d$}\\
\midrule
emotions            & $0.81  $ & $\pm 0.01$  & $0.81  $ & $\pm 0.014$ & $0.80  $ & $\pm 0.013$ & $0.81 $ & $\pm 0.014$ \\
scene               & $0.873 $ & $\pm 0.004$ & $0.876 $ & $\pm 0.003$ & $0.877 $ & $\pm 0.007$ & $0.878 $ & $\pm 0.006$ \\
yeast               & $0.757 $ & $\pm 0.008$ & \dotuline{$0.746 $} & $\pm 0.004$ & $0.752 $ & $\pm 0.009$ & $0.757 $ & $\pm 0.01$ \\
tmc2017             & $0.782 $ & $\pm 0.003$ & \dotuline{$0.759 $} & $\pm 0.004$ & \dotuline{$0.77  $} & $\pm 0.002$ & $0.779 $ & $\pm 0.002$ \\
genbase             & $0.987 $ & $\pm 0.005$ & $0.991 $ & $\pm 0.004$ & $0.992 $ & $\pm 0.001$ & $0.992 $ & $\pm 0.005$ \\
reuters             & $0.88  $ & $\pm 0.003$ & $0.88  $ & $\pm 0.003$ & $0.878 $ & $\pm 0.004$ & $0.88  $ & $\pm 0.004$ \\
medical             & $0.855 $ & $\pm 0.008$ & {$\bf 0.867 $} & $\pm 0.009$ & {$\bf 0.872 $} & $\pm 0.006$ & $0.862 $ & $\pm 0.008$ \\
enron               & $0.66  $ & $\pm 0.01$  & $0.65  $ & $\pm 0.01 $ & $0.663 $ & $\pm 0.008$ & $0.66  $ & $\pm 0.01$ \\
mediamill           & $0.786 $ & $\pm 0.002$ & \dotuline{$0.778 $} & $\pm 0.002$ & \dotuline{$0.781 $} & $\pm 0.002$ & $0.784 $ & $\pm 0.001$ \\
Yeast-GO            & $0.49  $ & $\pm 0.009$ & \dotuline{$0.47  $} & $\pm 0.01 $ & $0.482 $ & $\pm 0.008$ & $0.48  $  & $\pm 0.01$ \\
bibtex              & $0.584 $ & $\pm 0.005$ & \dotuline{$0.538 $} & $\pm 0.005$ & \dotuline{$0.564 $} & $\pm 0.004$ & $0.583 $ & $\pm 0.004$ \\
CAL500              & $0.5   $ & $\pm 0.007$ & $0.502 $ & $\pm 0.008$ & $0.499 $ & $\pm 0.007$ & $0.503 $ & $\pm 0.009$ \\
WIPO                & $0.52  $ & $\pm 0.01 $ & \dotuline{$0.474 $} & $\pm 0.007$ & \dotuline{$0.49  $} & $\pm 0.01 $ & $0.515 $ & $\pm 0.006$ \\
EUR-Lex (subj.)     & $0.845 $ & $\pm 0.006$ & \dotuline{$0.834 $} & $\pm 0.004$ & \dotuline{$0.838 $} & $\pm 0.003$ & $0.845 $ & $\pm 0.005$ \\
bookmarks           & $0.453 $ & $\pm 0.002$ & \dotuline{$0.436 $} & $\pm 0.002$ & \dotuline{$0.444 $} & $\pm 0.003$ & $0.452 $ & $\pm 0.002$ \\
diatoms             & $0.73  $ & $\pm 0.01 $ & \dotuline{$0.69  $} & $\pm 0.01$  & \dotuline{$0.71  $} & $\pm 0.01 $ & $0.73  $ & $\pm 0.01 $ \\
corel5k             & $0.285 $ & $\pm 0.009$ & {$\bf 0.313 $} & $\pm 0.011$ & {$\bf 0.309 $} & $\pm 0.009$ & $0.285 $ & $\pm 0.011$ \\
EUR-Lex (dir.)      & $0.815 $ & $\pm 0.007$ & \dotuline{$0.805 $} & $\pm 0.006$ & \dotuline{$0.807 $} & $\pm 0.009$ & $0.815 $ & $\pm 0.007$ \\
SCOP-GO             & $0.778 $ & $\pm 0.005$ & $0.782 $ & $\pm 0.004$ & $0.782 $ & $\pm 0.006$ & $0.778 $ & $\pm 0.005$ \\
delicious           & $0.354 $ & $\pm 0.003$ & {$\bf 0.36  $} & $\pm 0.004$ & {$\bf 0.358 $} & $\pm 0.004$ & $0.355 $ & $\pm 0.003$ \\
drug-interaction    & $0.353 $ & $\pm 0.011$ & {$\bf 0.375 $} & $\pm 0.017$ & $0.364 $ & $\pm 0.014$ & $0.355 $ & $\pm 0.016$ \\
protein-interaction & $0.299 $ & $\pm 0.013$ & $0.307 $ & $\pm 0.009$ & $0.305 $ & $\pm 0.012$ & $0.306 $ & $\pm 0.017$ \\
Expression-GO       & $0.231 $ & $\pm 0.007$ & \dotuline{$0.218 $} & $\pm 0.005$ & $0.228 $ & $\pm 0.005$ & $0.235 $ & $\pm 0.005$ \\
\bottomrule
\end{tabular}
\end{footnotesize}
\label{tab:extra_trees_results_summary}
\end{table*}

\section{Input vs output space randomization on ``Delicious''}

In this section, we carry out the same experiment as in Section 4.4 of
the main paper but on the ``Delicious''
dataset. Figure~\ref{fig:k_delicious_gaussian} shows the evolution of
the accuracy for growing values of $k$ (i.e. decreasing strength of
the input space randomization), for three different values of $m$ (in
this case $m \in \left\{1, \ln{d}, 2 \ln{d}\right\}$) on a Gaussian
output space.

Like on ``Drug-interaction'' (see Figure 2 of the paper), using
low-dimensional output spaces makes the method more robust with
respect to over-fitting as $k$ increases. However, unlike on
``Drug-interaction'', it is not really possible to improve over
baseline Random Forests by tuning jointly input and output
randomization. This shows that the interaction between $m$ and $k$ may
be different from one dataset to another.

\begin{figure}[htb]
\centering
\includegraphics[width=0.45\textwidth]{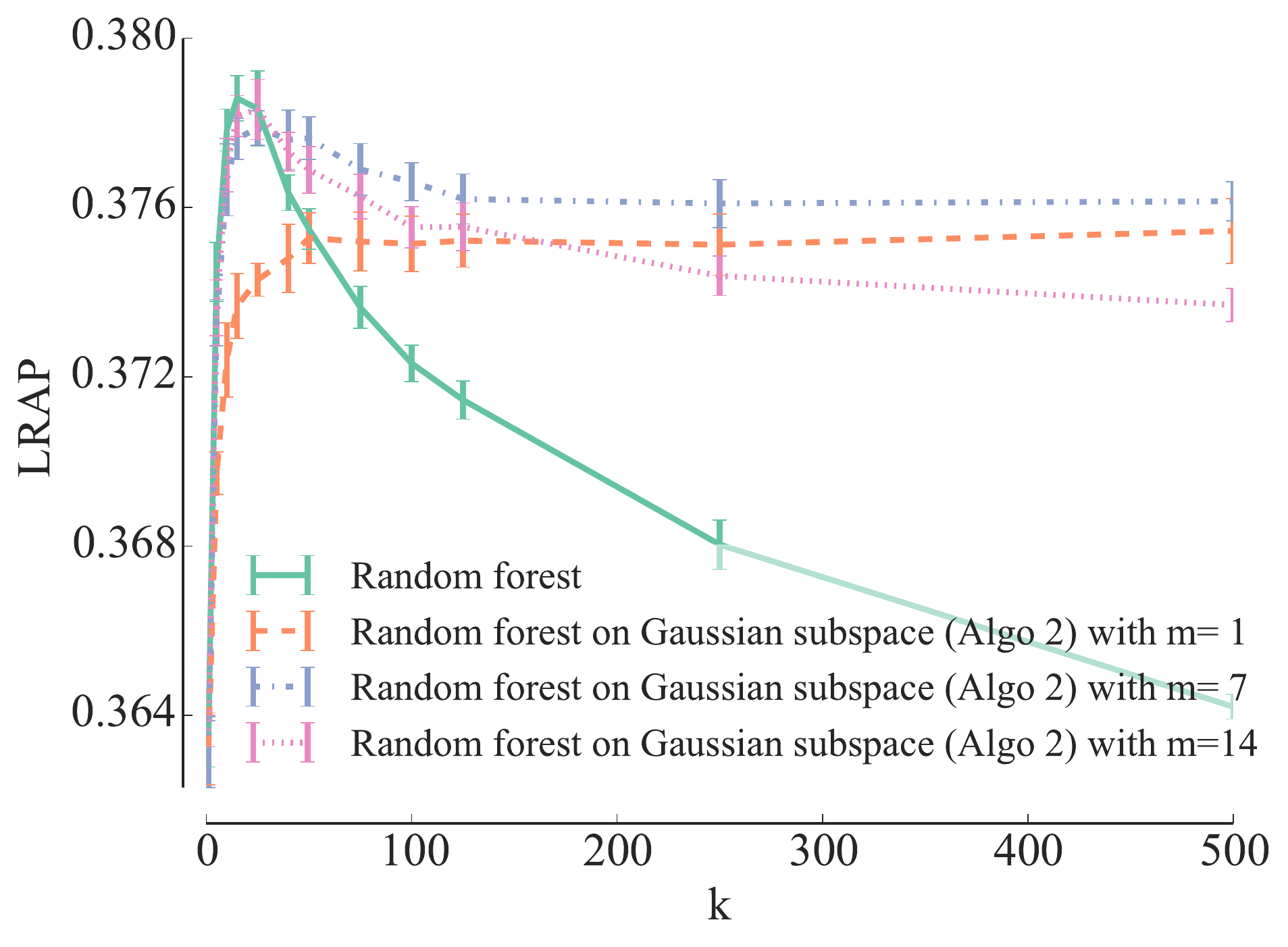}
\caption{``Delicious''  dataset: $n_{\min}=1$; $t=100$.}
\label{fig:k_delicious_gaussian}
\end{figure}

\end{document}